\documentclass{article}

\usepackage[preprint]{neurips_2026}

\usepackage{algorithm}
\usepackage{algpseudocode}
\usepackage{xspace}
\usepackage{booktabs}
\usepackage[dvipsnames]{xcolor}

\usepackage{amsmath}

\usepackage{booktabs}     
\usepackage{graphicx}

\usepackage[dvipsnames]{xcolor}

\usepackage[utf8]{inputenc} 
\usepackage[T1]{fontenc}    
\usepackage{hyperref}       
\usepackage{url}            
\usepackage{booktabs}       
\usepackage{amsfonts}       
\usepackage{nicefrac}       
\usepackage{microtype}      
\usepackage{xcolor}         

\usepackage{algorithmicx}

\usepackage{amsmath}
\usepackage{amssymb}
\usepackage{mathtools}
\usepackage{amsthm}
\usepackage[table]{xcolor}

\newcommand{\CAST}{\textit{CAST}\xspace}
\newcommand{\CASTone}{\textit{CAST-1}\xspace}
\newcommand{\CASTtwo}{\textit{CAST-2}\xspace}
\newcommand{\AC}{\textit{AC}\xspace}
\newcommand{\ACtwo}{\textit{AC2}\xspace}

\newcommand{\calO}{\mathcal{O}}
\newcommand{\bigO}{\calO}

\theoremstyle{plain}
\newtheorem{theorem}{Theorem}[section]
\newtheorem{proposition}[theorem]{Proposition}
\newtheorem{lemma}[theorem]{Lemma}
\newtheorem{corollary}[theorem]{Corollary}

\newtheorem{question}[theorem]{Question}
\newtheorem*{question*}{Question}

\renewcommand{\Pr}[1]{\mathbb{P}\left(#1\right)}

\title{CAST: Canonical Approximate Schur Tree for Approximate Cholesky on Graphs}

\author{
  Meher Chaitanya\\
  \texttt{mcpi@kth.se}\\
  KTH Royal Institute of Technology\\
  Stockholm, Sweden\\
  \And
  Cameron Musco\\
  \texttt{cmusco@cs.umass.edu}\\
  University of Massachusetts Amherst\\
  Amherst, Massachusetts, USA\\
  \AND
   Aristides Gionis\\
\texttt{argioni@kth.se}\\
KTH Royal Institute of Technology\\
  Stockholm, Sweden\\
}

\begin{document}

\maketitle

\begin{abstract}

Graph-data workloads such as diffusion estimation, ranking, semi-supervised learning, and network optimization often solve many Laplacian or symmetric diagonally dominant M-matrix (SDDM) systems with the same coefficient matrix. Approximate Cholesky preconditioners eliminate vertices one at a time and store the resulting sparse approximate factorization, the \emph{factor}, whose construction cost is amortized across these solves. But eliminating a vertex, the \emph{pivot}, creates a dense
Schur-complement clique among its $d$ active neighbors. We introduce \CAST{} (Canonical Approximate Schur Tree), which replaces this clique with a weighted random spanning tree sampled directly from it. Every realization is connected and contains exactly $d-1$ edges, while reweighting each selected edge by the reciprocal of its tree-inclusion probability makes the update unbiased. The distribution is independent of the ordering of the pivot neighbors, and we prove that its leverage-score marginals minimize the largest normalized reweighted-edge contribution among unbiased inverse-marginal one-tree estimators.

We also introduce \CAST-$\rho$, which replaces each pivot neighbor with $\rho$ copies, each carrying a $1/\rho$ share of that neighbor's incident weight, samples a weighted random spanning tree on the expanded clique, and contracts the copies back to the original neighborhood. The resulting update remains unbiased and connected, can be sampled exactly in $O(\rho d)$ time, and satisfies a $1/\rho$ bound on the second moment of the normalized local Schur
error. Increasing $\rho$ therefore reduces certified local sampling
variability, but may increase construction cost and downstream fill.

We evaluate \CAST{} on $201$ matrix systems, including $168$ Newton-step Laplacians from maximum-flow interior-point methods, constructing one factor and reusing it for $250$ right-hand sides.  \CAST-$\rho$ is $1.11$--$4.43\times$ faster than the better of the state-of-the-art \AC{} and \ACtwo{} baselines~\cite{AC}. Overall, we observe that \CASTone{} is the faster default, whereas \CASTtwo{} is preferable when its additional edge contributions remain inexpensive.

\end{abstract}

\maketitle

\section{Introduction}
Many graph-data and network-optimization pipelines repeatedly solve linear systems defined by a fixed graph operator. 
Examples include ranking and diffusion estimation~\cite{page1999pagerank,andersen2006local}, label propagation~\cite{zhu2003semi,zhou2003learning}, shift-and-invert spectral graph methods~\cite{huang2019isira}, and network optimization~\cite{spielman2004nearly,spielman2008graph}. 
These applications lead to Laplacian or symmetric diagonally dominant M-matrix (SDDM) systems in which the operator is \emph{reused} across many right-hand sides. The relevant computational objective is therefore not only to solve one system quickly, but to construct a pre\-conditioner whose cost can be amortized across repeated solves.

Approximate Cholesky factorization provides a practical approach to this problem. As in sparse Gaussian elimination, vertices are eliminated sequentially, but the dense fill created by exact
elimination is replaced by sparse randomized
updates~\cite{kyng2016approximate,AC}. The resulting approximate
factor is used as a conjugate-gradient pre\-conditioner and reused
across multiple right-hand sides. Eliminating a vertex $v$ of degree $d$,
with incident edge weights $a_1,\ldots,a_d$ and total weight
$a=\sum_{i=1}^{d} a_i$, creates the Schur-complement clique
\[
K_v \;=\; \sum_{1\leq i<j\leq d}
\frac{a_i a_j}{a}\,
(\mathbf e_i-\mathbf e_j)
(\mathbf e_i-\mathbf e_j)^\top,
\]
where $\mathbf e_i$ denotes the standard basis vector of the
$i$-th neighbor. This clique contains $\Theta(d^2)$ edges. Approximate Cholesky methods avoid this quadratic fill by replacing $K_v$ with a random surrogate containing only $\bigO(d)$ edges. 

Kyng and Sachdeva sample clique edges independently and control the accumulated error through a matrix-martingale analysis~\cite{kyng2016approximate}. The practical \AC{} and \ACtwo{} solvers~\cite{AC} take a complementary approach: every sampled local update has connected \emph{support}, the graph on the $d$ pivot neighbors formed by its nonzero-weight edges. A graph Laplacian has nullspace
$\operatorname{span}\{\mathbf 1\}$ exactly when its support is connected,
matching the exact Schur clique. Since a spanning tree is the sparsest connected support, with exactly $d-1$ edges, we first study surrogates consisting of a single spanning tree $T$ of the clique. 

For a clique edge $e$ with weight $w_e$, let $p_e=\Pr{e\in T}>0$ be its
inclusion probability, and assign $e$ the weight $w_e/p_e$ whenever it is
selected. This \emph{inverse-marginal reweighting} makes the surrogate
unbiased, since the expected contribution of $e$ is $p_e(w_e/p_e)=w_e$. The remaining choice is the distribution over trees. If $p_e$ is small, edge $e$ receives a correspondingly large weight when selected. Let $\tau_e$ be the leverage score of $e$ in the clique, equal to its weight times its effective resistance within $K_v$~\cite{spielman2008graph}. The normalized size of the reweighted contribution---its size measured
relative to the quadratic form of $K_v$ itself---is exactly $\tau_e/p_e$; see Section~\ref{sec:preliminaries}. A large ratio lets one selected edge dominate the local update and inflates the error bounds obtained from matrix-concentration analyses~\cite{kyng2016approximate}. We therefore ask:

\begin{question}\label{que:1}
Among unbiased inverse-marginal one-tree replacements of a Schur clique, which edge-inclusion probabilities minimize the worst normalized contribution $\max_e \tau_e/p_e$?
\end{question}

The connected updates of \AC{} and \ACtwo{} are generated through randomized sequential edge pairing~\cite{AC}. Although these methods perform strongly across many SDDM systems, their induced sampling distributions depend on the order in which the incident edges are processed and are not designed to minimize $\max_e\tau_e/p_e$. The distinguishing feature of our method is therefore not connectivity alone, but the distribution used to generate the local update.

We introduce \emph{\CAST} (Canonical Approximate Schur Tree): at each
elimination step, it samples a weighted random spanning tree of the Schur clique and applies inverse-marginal reweighting. By the transfer-current theorem~\cite{lyons2003determinantal}, each clique edge is then included with probability equal to its leverage score, $p_e=\tau_e$. This answers Question~\ref{que:1}: the leverage-score marginals are the unique locally minimax-optimal choice among unbiased inverse-marginal one-tree estimators (Theorem~\ref{thm:minimax}).  We call the sampling rule \emph{canonical} because it depends only on the weighted Schur clique and not on the ordering of the pivot
neighbors.

We further introduce \CAST-$\rho$, with integer parameter $\rho\geq 1$, which splits each pivot neighbor into $\rho$ equal-share auxiliary copies, samples one weighted random spanning tree from the expanded Schur clique, and contracts the copies back to the original neighbors.  For $\rho=1$ it
coincides with \CAST{}; for $\rho>1$ the contracted update need not be a
tree, but it remains unbiased and connected and satisfies a local
second-moment bound of order $1/\rho$, reducing sampling variability at
the cost of additional construction work and possible downstream fill. We
write \CASTone{} and \CASTtwo{} for the $\rho=1$ and $\rho=2$
configurations used in our evaluation.

Our guarantees concern the estimator associated with a single Schur-clique replacement; we do not claim a worst-case spectral guarantee for the complete approximate factorization.

In summary, we make the following contributions:

\begin{enumerate}
\item \textbf{Canonical tree sampling and its splitting extension.}
We introduce \CAST{} and \CAST-$\rho$. At a degree-$d$ pivot, a tree is sampled exactly from the target weighted spanning-tree distribution and contracted in $\bigO(\rho d)$ time using weighted Pr\"ufer codes, without materializing either the original or the expanded dense Schur clique.

\item \textbf{Local minimax and second-moment guarantees.}
For the base $\rho=1$ update, the leverage-score vector is the unique
marginal vector minimizing the largest normalized reweighted-edge
contribution among unbiased inverse-marginal one-tree estimators
(Theorem~\ref{thm:minimax}). For general $\rho$, \CAST-$\rho$ remains
unbiased and connected and satisfies a normalized local second-moment bound of order $1/\rho$, quantifying the reduction in sampling variability obtained through splitting (Section~\ref{sec:local-theory}).

\item \textbf{Reuse-aware evaluation.}
We evaluate \CAST{} by constructing one preconditioner and reusing it across many right-hand sides. Across heterogeneous SDDM systems and max flow IPM sequences, our experiments identify when \CAST{} reduces repeated-solve cost relative to \AC{} and \ACtwo{}. We further characterize when increasing $\rho$ improves convergence and when its additional edge contributions
create excessive downstream fill.
\end{enumerate}

\paragraph{Independent contemporaneous work.}
We recently became aware of independent work by Baumann, Kyng, and Zöcklein~\cite{baumann2026vac}, who propose, under the name VAC, the same base $\rho=1$ local sampling rule as CAST-1: a weighted random spanning tree of the product-form Schur clique sampled via a weighted Pr"ufer code and equipped with inverse-marginal reweighting. Their work was developed independently of ours and establishes linear-time sampling rule. Our work additionally gives a local minimax characterization of the leverage-score marginals, introduces the CAST-$\rho$ splitting construction and its $1/\rho$ local second-moment guarantee, and evaluates these estimators as reusable preconditioners across a broad collection of SDDM systems.

\section{Preliminaries and Problem Setting}
\label{sec:preliminaries}
We formalize SDDM linear systems, the factor-reuse
setting, and the local Schur-clique replacement problem underlying~\CAST. The preliminaries on Pr\"ufer codes are provided in Appendix~\ref{app:prufer-background}.

\paragraph{\textbf{Graph Laplacians and SDDM systems}.}

Let $G=(V,E,w)$ be an undirected weighted graph with $n=|V|$ vertices. For an oriented edge $e=(u,v)$ write $\mathbf{b}_e=\mathbf{e}_u-\mathbf{e}_v$, where $\mathbf{e}_u$ is the $u$-th standard basis
vector. The Laplacian of $G$ is $L_G=\sum_{e\in E} w_e\, \mathbf{b}_e \mathbf{b}_e^\top$. If $G$ is connected, then $\ker(L_G)=\operatorname{span}\{\mathbf 1\}$.
Since $L_G$ is symmetric,
$\operatorname{range}(L_G)=\ker(L_G)^\perp=\mathbf 1^\perp$, so the system
$L_G\mathbf x=\mathbf b$ is solvable if and only if
$\mathbf b\perp\mathbf 1$. 

Throughout, $A$ denotes the coefficient matrix of the system to be
solved: either a graph Laplacian $L_G$ or a symmetric diagonally
dominant M-matrix (\emph{SDDM matrix}), i.e.\ a symmetric
positive-definite matrix with nonpositive off-diagonal entries
satisfying $A_{ii}\geq\sum_{j\neq i}|A_{ij}|$.
 Equivalently, an SDDM matrix can be written as
$A=L_G+\Gamma$, where $L_G$ is a graph Laplacian and
$\Gamma\succeq 0$ is a diagonal matrix. The matrix $A$ is positive definite
provided that $\Gamma$ has at least one positive entry on each
connected component of $G$. This class of matrices includes many shifted and regularized graph operators.  For example, let \(W \in \mathbb{R}^{n \times n}\) be a symmetric, nonnegative weighted adjacency matrix, let $D = \operatorname{diag}(W\mathbf{1})$ be its degree matrix, and let $L = D-W$ be a graph Laplacian. For a damping parameter $\alpha \in (0,1)$ and a seed distribution $\mathbf{s} \in \mathbb{R}^{n}$, the personalized PageRank vector \(\mathbf{x}\) satisfies $(D-\alpha W)\mathbf{x}=(1-\alpha)D\mathbf{s}$. The coefficient matrix is SDDM because
$D-\alpha W =\alpha(D-W)+(1-\alpha)D =\alpha L+(1-\alpha)D$  is symmetric, has non\-positive off-diagonal entries, and is diagonally dominant~\cite{gleich2015pagerank}. If every vertex has positive weighted degree, the matrix is also positive definite. In repeated-query settings, the operator $D-\alpha W$ remains fixed while the seed distribution $\mathbf{s}$ varies across queries.

\paragraph{\textbf{The factor-reuse setting}.}
Many applications solve a sequence of linear systems $A\,\mathbf{x}^{(r)}=\mathbf{b}^{(r)}$, for $r=1,\ldots,q$, in which the coefficient matrix $A$ remains fixed while the right-hand side changes: diffusion from different seed sets, label propagation across classes, effective-resistance queries for many
source--sink pairs, and inner systems in iterative optimization.
An approximate Cholesky method performs the elimination once and stores the result as a sparse approximate factorization. Recall that the exact Cholesky decomposition of a positive-definite matrix is
$A=\mathcal{C}\mathcal{C}^\top$ with $\mathcal{C}$ lower triangular;
approximate elimination instead produces a sparse
$\widehat{\mathcal{C}}$ with
$\widehat{\mathcal{C}}\widehat{\mathcal{C}}^\top\approx A$, which we call the \emph{factor}. The factor is stored implicitly as the elimination data recorded at each pivot and is applied through forward and backward substitutions as a preconditioner for each right-hand side. The factor is constructed at cost $T_{\mathrm{build}}$ and reused for all
$q$ solves. Writing $\overline{T}_{\mathrm{solve}}$ for the average time of one preconditioned-CG solve using the factor, the total and amortized workload costs are

\begin{equation*}
  T_q^{\mathrm{total}}
  = T_{\mathrm{build}}
    + q\,\overline{T}_{\mathrm{solve}},
  \text{ and }
  T_q^{\mathrm{amort}}
  = \frac{T_q^{\mathrm{total}}}{q}
  = \frac{T_{\mathrm{build}}}{q}
    + \overline{T}_{\mathrm{solve}}.
\end{equation*}
Thus $q=1$ measures one-shot performance, while increasing $q$
amortizes construction and captures the benefit of factor reuse.

\paragraph{\textbf{Exact elimination and Schur cliques}}
For the local elimination analysis, we work with a Laplacian
representation; general SDDM systems can be reduced to this form
using the standard Gremban expansion~\cite{gremban1996combinatorial}. Approximate Cholesky methods eliminate vertices one at a time.
Suppose a pivot $v$ has active neighbors $u_1,\ldots,u_d$ with
$a_i=w_{vu_i}$ and $a=\sum_i a_i$. Eliminating $v$ exactly removes
the star at $v$ and adds the \emph{Schur clique}
\begin{equation}\label{eqn:pivot_clique}
 K_v = \sum_{1\le i<j\le d}\frac{a_i a_j}{a}\,
  (\mathbf{e}_i-\mathbf{e}_j)(\mathbf{e}_i-\mathbf{e}_j)^\top
\end{equation}
to the remaining graph, where $\mathbf{e}_i$ denotes the $i$-th standard basis vector in the coordinate system indexed by $u_1,\ldots,u_d$; edges already present among the neighbors simply accumulate the corresponding weight. The clique contains $\binom{d}{2}=\Theta(d^2)$ edges, even when the original graph is sparse. To avoid ``filling up'' the matrix, approximate Cholesky methods replace $K_v$ by a sparse randomly-sampled update $\widehat K_v$ satisfying $\mathbb{E}[\widehat K_v] = K_v$~\cite{kyng2016approximate,AC}.  
The \AC and \ACtwo solvers construct unbiased connected local updates by randomized edge pairing, with \ACtwo doubling the per-edge sampling budget for robustness~\cite{AC}. These sampling distributions, which have been empirically studied, depend on the order in which the neighbor edges are processed. In contrast, the sampling distribution in \CAST is determined by the weighted Schur clique and admits the local minimax and second-moment guarantees established in Section~\ref{sec:local-theory}.

\paragraph{\textbf{The local Schur-clique replacement problem}.}
We write $K^+$ for the Moore--Penrose pseudoinverse and
$K^{+/2}=(K^+)^{1/2}$. Let
$K=\sum_{e\in E_K} w_e\,\mathbf b_e\mathbf b_e^\top$ be a connected
\emph{local} graph Laplacian: a $d\times d$ matrix indexed by the $d$
active neighbors of a pivot rather than by all $n$ vertices, with edge
set $E_K$ and weights $w_e>0$. Its \emph{support graph} is $([d],E_K)$,
the graph formed by its positive-weight edges; since it is connected,
$\ker(K)=\operatorname{span}\{\mathbf 1\}$ and
$\operatorname{rank}(K)=d-1$. We measure a positive-semidefinite local contribution $H$ relative to $K$
by $\lVert K^{+/2}HK^{+/2}\rVert=\max_{\mathbf x\perp\mathbf 1}
(\mathbf x^\top H\mathbf x)/(\mathbf x^\top K\mathbf x)$, the largest ratio of $H$-energy to $K$-energy, and call this normalization the \emph{energy geometry} of $K$. The \emph{leverage score} of edge $e$ is $\tau_e=w_e\,\mathbf b_e^\top K^+\mathbf b_e$, which is exactly the size of that edge's own term $w_e\mathbf b_e\mathbf b_e^\top$ in this geometry~\cite{spielman2008graph,kyng2016approximate}; leverage scores satisfy $\sum_e\tau_e=\operatorname{rank}(K)=d-1$.

For a distribution $\mathcal D$ over spanning trees of the support graph of $K$, let $p_e=\mathbb{P}_{T\sim\mathcal D}(e\in T)$, and assume $p_e>0$ for every $e\in E_K$. The inverse-marginal estimator
\begin{equation*}
  \widehat K_{\mathcal{D}}(T)=\sum_{e\in T}\frac{w_e}{p_e}\,
  \mathbf b_e\mathbf b_e^\top
\end{equation*}
is unbiased, since
$\mathbb E[\widehat K_{\mathcal D}(T)]
 =\sum_e p_e(w_e/p_e)\,\mathbf b_e\mathbf b_e^\top=K$. We measure the
\emph{local risk} of $\mathcal D$ by the largest normalized contribution
of any sampled edge,
\begin{equation}\label{eq:risk}
R_K(\mathcal{D}) = \max_{e\in E_K}\left\lVert K^{+/2}\!\left(
\frac{w_e}{p_e}\mathbf b_e\mathbf b_e^\top\right)\!K^{+/2}\right\rVert
= \max_{e\in E_K}\frac{\tau_e}{p_e},
\end{equation}
where the second equality holds because each normalized edge
contribution is rank one. Question~\ref{que:1} is then the following
general design problem: among such $\mathcal D$, which minimizes
$R_K(\mathcal D)$? Since $R_K$ depends on $\mathcal D$ only through its
marginals $(p_e)_{e\in E_K}$, this is equivalently a question about the
feasible marginal vectors of spanning-tree distributions. \CAST{} uses
the weighted random spanning-tree distribution of $K$, whose marginals
satisfy $p_e=\tau_e$ by the transfer-current
theorem~\cite{lyons2003determinantal}; this equalizes the normalized size
of every sampled edge and, as we show in Theorem~\ref{thm:minimax},
minimizes $R_K$.


\section{\CAST: Canonical Approximate Schur Trees}
\label{sec:cast}

In this section we define the local \CAST-$\rho$ update and the approximate-Cholesky pre\-conditioner obtained by applying it at each
pivot. The construction replaces the Schur clique of Eq.~\eqref{eqn:pivot_clique} with a sparse random update derived from a weighted spanning tree on a temporarily expanded neighborhood.

\subsection{The \CAST{} update}
\label{sec:cast-update}

Consider an elimination step with pivot $v$, active neighbors
$u_1,\ldots,u_d$, and incident edge weights $a_i=w_{vu_i}>0$, where
$a=\sum_{i=1}^{d}a_i$. We interpret these edge weights as electrical
conductances, and refer to the pivot neighbors as \emph{terminals}.
Exact elimination of $v$ adds the Schur clique $K_v$ of
Eq.~\eqref{eqn:pivot_clique}, which contains $\Theta(d^2)$ edges.

For an integer splitting factor $\rho\geq 1$, \CAST-$\rho$ temporarily
replaces each terminal $u_i$ by $P_i=\{(i,1),\ldots,(i,\rho)\}$, a block of $\rho$ \emph{auxiliary
copies}, each carrying conductance
$c_p=a_i/\rho$. Let $P=\bigcup_{i=1}^{d}P_i$, and let $\phi(p)=i$
denote the terminal associated with $p\in P_i$, so that
$\sum_{p\in P_i}c_p=a_i$ and $\sum_{p\in P}c_p=a$. On $P$,
\CAST-$\rho$ defines the complete weighted \emph{expanded Schur clique}
with edge conductances $w^{\mathrm{exp}}_{pq}=c_pc_q/a$ for $p\neq q$,
and samples a weighted random spanning tree $T$ according to
\begin{equation}
\label{eqn:expanded-wrst}
    \mathbb{P}(T)\propto \prod_{(p,q)\in T}w_{pq}^{\mathrm{exp}}.
\end{equation}
Sampling from such a distribution is not straightforward for general
edge weights; Section~\ref{sec:prufer} shows that the product form of
$w^{\mathrm{exp}}_{pq}$ admits exact sampling in $\bigO(\rho d)$ time.
The expanded clique is therefore used only to define the distribution
and is never materialized explicitly.

Each sampled auxiliary edge $(p,q)\in T$ is assigned conductance
$c_pc_q/(c_p+c_q)$, which is precisely its inverse-marginal reweighting
on the expanded clique (see Lemma~\ref{lem:aux-marginals}). The
auxiliary-copy blocks are then contracted back to their terminals. If
$\phi(p)=i\neq j=\phi(q)$, the edge $(p,q)$ contributes this
conductance to the terminal edge $(u_i,u_j)$; if $\phi(p)=\phi(q)$, it
becomes a self-loop after contraction and contributes nothing to the
terminal Laplacian. Writing $\widehat K_v$ for the random Laplacian on
the pivot neighborhood obtained after contracting the blocks and
discarding self-loops,
\begin{equation}
\label{eqn:cast-update}
\widehat K_v = \sum_{\substack{(p,q)\in T\\ \phi(p)\neq\phi(q)}}
\frac{c_pc_q}{c_p+c_q}
\bigl(\mathbf e_{\phi(p)}-\mathbf e_{\phi(q)}\bigr)
\bigl(\mathbf e_{\phi(p)}-\mathbf e_{\phi(q)}\bigr)^\top,
\end{equation}
where $\mathbf e_i$ denotes the basis vector associated with terminal
$u_i$. The auxiliary tree contains exactly $\rho d-1$ edges, so after
discarding within-block edges and aggregating parallel contracted
edges, the update contains at most $\rho d-1$ nonzero terminal-edge
contributions. For $\rho>1$ it need not itself be a tree on the
terminals, although its support is always connected (Theorem~\ref{thm:cast-connected}).

\paragraph{\CASTone: the base $\rho=1$ update.}
When $\rho=1$, each block holds a single copy carrying its terminal's
full conductance, so the expanded clique is the Schur clique itself and
contraction is the identity: \CAST-$\rho$ reduces to sampling one
weighted random spanning tree of $K_v$. Because Schur complementation
preserves effective resistances among the retained vertices, the
effective resistance between terminals $u_i$ and $u_j$ in $K_v$ equals
the series resistance of the path $u_i$--$v$--$u_j$ before elimination,
$R_{ij}=1/a_i+1/a_j$. The leverage score of edge $(i,j)$ is,
$
  \tau_{ij} = \frac{a_ia_j}{a}\left(\frac{1}{a_i}+\frac{1}{a_j}\right)
  = \frac{a_i+a_j}{a}.
$
By the transfer-current theorem the weighted random spanning tree selects edge $(i,j)$ with
probability $\tau_{ij}$, and inverse-marginal reweighting assigns it
conductance
$\frac{a_ia_j/a}{\tau_{ij}}=\frac{a_ia_j}{a_i+a_j}$,

exactly the series conductance of the two-hop path $u_i$--$v$--$u_j$
that the edge replaces. Thus \CASTone{} is an order-independent
one-tree replacement whose edge marginals equal the clique leverage
scores; Theorem~\ref{thm:minimax} shows that these marginals uniquely
minimize the local risk $R_{K_v}$.

\subsection{Exact sampling in $\bigO(\rho d)$ time via Pr\"ufer codes}
\label{sec:prufer}

Sampling the update requires a weighted random spanning tree of the
expanded Schur clique, a complete graph on $m=\rho d$ auxiliary copies.
Its conductances have the product form
$w^{\mathrm{exp}}_{pq}=c_pc_q/a$, and this structure permits exact
sampling from~\eqref{eqn:expanded-wrst} using an i.i.d.\ weighted
Pr\"ufer code, without materializing the dense clique. Background on
Pr\"ufer codes is provided in
Appendix~\ref{app:prufer-background}.

Fix a spanning tree $T$ on the auxiliary-copy set $P$. Since $T$ has $m-1$ edges,
\[
  \prod_{(p,q)\in T} w_{pq}^{\mathrm{exp}}
  = a^{-(m-1)}\prod_{p\in P}c_p^{\deg_T(p)},
\]
so, up to a factor independent of $T$, the target probability is
determined entirely by the vertex degrees. For $m\geq 2$ the Pr\"ufer
correspondence is a bijection between labeled trees on $P$ and
sequences in $P^{m-2}$, under which $p$ appears exactly
$\deg_T(p)-1$ times in the code of
$T$~\cite{aigner1999proofs,west2001introduction}. Drawing the $m-2$
symbols independently with $\Pr{\text{symbol}=p}=c_p/a$ therefore
generates $T$ with probability
\[
  \prod_{p\in P}\left(\frac{c_p}{a}\right)^{\deg_T(p)-1}
  = a^{-(m-2)}\prod_{p\in P}c_p^{\deg_T(p)-1}.
\]
The ratio of this probability to the unnormalized target weight is
$a\prod_{p\in P}c_p^{-1}$, independent of $T$. Hence the i.i.d.\
Pr\"ufer construction samples exactly from the weighted random
spanning-tree distribution of the expanded Schur clique.

The sampler needs neither the expanded clique nor an explicit array of
auxiliary-copy weights. All copies in $P_i$ carry conductance
$a_i/\rho$, so each symbol can be drawn hierarchically: sample terminal
$i$ with probability $a_i/a$, then one of its $\rho$ copies uniformly.
An alias table for the terminal distribution is built once in
$\bigO(d)$ time, after which each of the $m-2$ symbols costs
$\bigO(1)$. Decoding the code into a tree takes $\bigO(m)$ time by the
standard leaf-pointer algorithm. Contraction is a single pass over the
$m-1$ auxiliary edges: cross-block edges are placed on the
corresponding terminal edge with conductance $c_pc_q/(c_p+c_q)$,
within-block edges become self-loops and are discarded, and parallel
terminal edges are aggregated by summing conductances. Sampling,
decoding, and contraction therefore cost
$\bigO(d+m)=\bigO(\rho d)$ in total.

\begin{proposition}[Sampling cost]
\label{prop:cost}
For every pivot with $d\geq 2$ active neighbors and every integer
$\rho\geq 1$, \textsc{CAST-SchurTree}$(a_1,\ldots,a_d,\rho)$
samples the exact weighted random spanning-tree distribution of the
expanded Schur clique and contracts the sampled tree onto the original
pivot neighborhood in $\bigO(\rho d)$ time, without materializing
either the $\Theta(d^2)$-edge Schur clique or the expanded clique.
\end{proposition}

\begin{proof}
    See Appendix~\ref{app:prop_cost}.
\end{proof}

Summing Proposition~\ref{prop:cost} over pivots, the elimination takes 
$\bigO\!\left(\rho\sum_v d_v\right)$ time, where $d_v$ is the degree of
$v$ in the residual graph when it is eliminated. The residual degrees
depend on the fill realized by earlier updates, which is why we study
the effect of $\rho$ on downstream fill empirically
(Section~\ref{sec:empirical_methods}).

Two structural features of the Schur clique make this possible. For
$\rho=1$, its star origin gives the leverage scores in closed form
(Section~\ref{sec:cast-update}), identifying the target distribution
without any resistance computation. Under $\rho$-way splitting, the
expanded clique retains product-form conductances, so the same
Pr\"ufer argument applies verbatim. Splitting thus refines the update while preserving both the algebraic structure needed for exact
sampling and the order-independence that makes the rule canonical.

\paragraph{The \CAST{} pre\-conditioner.}
Applying the local \CAST-$\rho$ update at every pivot yields an
approximate Cholesky factorization, \textsc{CAST-Chol}
(Algorithms~\ref{alg:cast-schur-tree}--\ref{alg:CAST}). The method
maintains a \emph{residual graph}, the weighted graph on the
not-yet-eliminated vertices, initialized to $A$. At each step it
selects a pivot $v$, records the \emph{pivot star}---the pivot, its
active neighbors, and their incident conductances, which together form
the column of $\widehat{\mathcal{C}}$ associated with $v$---then
deletes $v$ from the residual graph and inserts the contracted
terminal-edge contributions returned by Algorithm \textsc{CAST-SchurTree}. In
the reuse setting, the resulting factor is built once for the fixed
operator $A$ and applied within PCG for the right-hand sides
$\mathbf{b}^{(1)},\ldots,\mathbf{b}^{(q)}$.

\begin{algorithm}[t]
\caption{\textsc{CAST-SchurTree}$(a_1,\ldots,a_d,\rho)$}
\label{alg:cast-schur-tree}
\begin{algorithmic}[1]
\If{$d\leq 1$}
  \State \Return an empty terminal-edge multiset
\EndIf
\State $a\gets\sum_{i=1}^{d}a_i$, \quad $m\gets\rho d$
\State Build an alias table for $\pi_i=a_i/a$, $i\in[d]$
  \Comment{$\bigO(d)$}
\State Initialize an empty Pr\"ufer sequence $S$
\For{$t=1,\ldots,m-2$}
  \State Sample $i\sim\pi$ and $k\sim\operatorname{Unif}([\rho])$
  \State Append the auxiliary copy $p=(i,k)$ to $S$
    \Comment{$\phi(p)=i$}
\EndFor
\State Decode $S$ into its spanning tree $T$ on the auxiliary-copy
  set \Comment{$\bigO(m)$}
\State Initialize an empty multiset $\mathcal{E}$
\For{each auxiliary-tree edge $(p,q)\in T$ with
  $i=\phi(p)\neq\phi(q)=j$}
  \State Insert $\bigl(i,\,j,\;a_ia_j/(\rho(a_i+a_j))\bigr)$
    into $\mathcal{E}$
\EndFor
\State \Return $\mathcal{E}$
\end{algorithmic}
\end{algorithm}

\begin{algorithm}[t]
\caption{\textsc{CAST-Chol}$(A,\rho)$}
\label{alg:CAST}
\begin{algorithmic}[1]
\State Initialize the residual graph to $A$ and the factor $F$ to
  empty
\While{active vertices remain}
  \State Select the next pivot $v$ by the elimination-ordering rule
  \State Read its active neighbors $u_1,\ldots,u_d$ and incident
    conductances $a_1,\ldots,a_d$
  \State Record the pivot star of $v$ in $F$
  \State Delete $v$ and its incident edges from the residual graph
  \State $\mathcal{E}_v\gets
    \textsc{CAST-SchurTree}(a_1,\ldots,a_d,\rho)$
  \For{each $(i,j,h)\in\mathcal{E}_v$}
    \State Add conductance $h$ to the residual edge $(u_i,u_j)$
      \Comment{parallel contributions accumulate}
  \EndFor
\EndWhile
\State \Return $F$
\end{algorithmic}
\end{algorithm}

\section{Local Theory of the \CAST{} Update}
\label{sec:local-theory}

We analyze the \CAST-$\rho$ update at a single elimination step. For
the base $\rho=1$ update, we prove that the weighted spanning-tree
distribution uniquely minimizes the local risk defined in
Eq.~\eqref{eq:risk} among unbiased inverse-marginal
one-tree estimators. For general integer $\rho\geq 1$, we establish
unbiasedness, connected support, and a normalized local second-moment
bound of order $1/\rho$, quantifying the reduction in local sampling
variability obtained through finer splitting.

Fix a pivot $v$ with active neighbors $u_1,\ldots,u_d$, where
$d\geq 2$, and let $a_i=w_{vu_i}>0$, $a=\sum_{i=1}^{d}a_i$. Exact
elimination of $v$ creates the Schur clique $K_v$ of
\eqref{eqn:pivot_clique}. We identify the pivot neighborhood with
$[d]$, and all Loewner-order comparisons between local Laplacians on the pivot neighborhood are
understood on
$\mathbf{1}_d^\perp=\{\mathbf x\in\mathbb R^d:
\mathbf x^\top\mathbf 1_d=0\}$.

\paragraph{\textbf{Auxiliary-copy notation}.}
We use the notation of Section~\ref{sec:cast-update}: the blocks
$P_i=\{(i,1),\ldots,(i,\rho)\}$ of auxiliary copies with conductances
$c_p=a_i/\rho$, the set $P=\bigcup_{i=1}^{d}P_i$, the map $\phi(p)=i$,
the expanded Schur clique with conductances
$w^{\mathrm{exp}}_{pq}=c_pc_q/a$, and the tree $T$ drawn
from~\eqref{eqn:expanded-wrst}. For an auxiliary edge $e=(p,q)$ we
write $H_e=H_{pq}$ for its contracted contribution to the terminal
Laplacian: setting $i=\phi(p)$ and $j=\phi(q)$,
\[
H_{pq}=\frac{c_pc_q}{c_p+c_q}
(\mathbf e_i-\mathbf e_j)(\mathbf e_i-\mathbf e_j)^\top
\quad\text{if } i\neq j,
\qquad H_{pq}=0 \quad\text{if } i=j,
\]
where $\mathbf e_i$ is the basis vector of terminal $u_i$, so that
$\widehat K_v=\sum_{(p,q)\in T}H_{pq}$ as in~\eqref{eqn:cast-update}.
Throughout this section $p_e=\Pr{e\in T}$ denotes the inclusion
probability of an edge $e$ of the \emph{expanded} clique; when
$\rho=1$ the expanded clique is $K_v$ itself, and $p_e$ agrees with
the clique-edge marginals of Section~\ref{sec:preliminaries}.

\paragraph{\textbf{Unbiasedness}.}
We first show that contracting the sampled auxiliary tree preserves
the exact Schur update in expectation.

\begin{lemma}[Auxiliary-edge marginals]
\label{lem:aux-marginals}
For every auxiliary edge $(p,q)$ of the expanded Schur clique,
$\Pr{(p,q)\in T}=\frac{c_p+c_q}{a}$.
\end{lemma}

\begin{proof}
See Appendix~\ref{app:aux_marginals}.
\end{proof}

In particular, the conductance assigned to a selected auxiliary edge
is exactly its inverse-marginal reweighting:
\[
\frac{w^{\mathrm{exp}}_{pq}}{\Pr{(p,q)\in T}}
=\frac{c_pc_q/a}{(c_p+c_q)/a}
=\frac{c_pc_q}{c_p+c_q}.
\]

\begin{theorem}[Unbiasedness]
\label{thm:cast-unbiased}
For every pivot $v$ and every integer $\rho\geq 1$,
$\mathbb{E}[\widehat K_v]=K_v$.
\end{theorem}

\begin{proof}
See Appendix~\ref{app:thm_unbiasedness}.
\end{proof}

\paragraph{\textbf{Connected support and correct local nullspace}.}
Unbiasedness controls the local update in expectation. \CAST-$\rho$
also has a deterministic structural property: every realization has
connected support on the pivot neighborhood.

\begin{theorem}[Connected local support]
\label{thm:cast-connected}
For every integer $\rho\geq 1$, the terminal support graph of
$\widehat K_v$ is connected on $\{u_1,\ldots,u_d\}$.
\end{theorem}

\begin{proof}
See Appendix~\ref{app:thm_cast_connected}.
\end{proof}

\begin{corollary}[Correct local nullspace and rank]
\label{cor:correct-rank}
For every integer $\rho\geq 1$,
$\ker(\widehat K_v)=\operatorname{span}\{\mathbf 1_d\}$ and
$\operatorname{rank}(\widehat K_v)=d-1$.
\end{corollary}

\begin{proof}
See Appendix~\ref{app:correct-rank}.
\end{proof}

Connectivity alone does not distinguish \CAST{} from the practical
\AC{} and \ACtwo{} updates, which are also connected; the distinction
lies in the distribution from which the local update is drawn. We turn to that distribution.

\subsection{Local minimax optimality of the base
\texorpdfstring{$\rho=1$}{rho=1} update}
\label{sec:minimax}

When $\rho=1$, \CAST{} samples a weighted random spanning tree
directly from the Schur clique, so the update is an inverse-marginal
one-tree estimator (Section~\ref{sec:preliminaries}). We now answer Question~\ref{que:1}: which edge marginals minimize the
local risk $R_K(\mathcal D)=\max_{e\in E_K}\tau_e/p_e$, the largest
normalized contribution of any sampled edge? We answer it for an
arbitrary connected local Laplacian $K$; the Schur clique $K_v$ is the
case relevant to \CAST{}.

\begin{theorem}[Local minimax optimality of leverage-score marginals]
\label{thm:minimax}
For every spanning-tree distribution $\mathcal{D}$ satisfying
$p_e>0$ for all $e\in E_K$, $R_K(\mathcal{D})\geq 1$, with equality if
and only if $p_e=\tau_e$ for every $e\in E_K$. In particular, the
weighted random spanning-tree distribution of $K$, whose edge
marginals satisfy $p_e=\tau_e$, is locally minimax-optimal among
unbiased inverse-marginal one-tree estimators.
\end{theorem}

\begin{proof}
See Appendix~\ref{app:minimax}.
\end{proof}

The proof is a short averaging argument: every spanning tree has
$d-1$ edges, so $\sum_e p_e=d-1=\sum_e\tau_e$, and a maximum is at
least a weighted average.

Theorem~\ref{thm:minimax} identifies the optimal edge-marginal vector
uniquely, namely $p_e=\tau_e$; it does not assert uniqueness of the
full distribution over spanning trees. \CASTone{} realizes these
marginals with the weighted random spanning-tree distribution of the
clique, which for the Schur clique $K_v$ has the closed form derived
in Section~\ref{sec:cast-update}: edge $(i,j)$ is included with
probability $\tau_{ij}=(a_i+a_j)/a$ and reweighted to conductance
$a_ia_j/(a_i+a_j)$. Hence \CASTone{} attains $R_{K_v}=1$, the smallest
value achievable in this class.

\subsection{Local second-moment certificate for
\texorpdfstring{\CAST-$\rho$}{CAST-rho}}
\label{ssec:second-moment}

Theorem~\ref{thm:minimax} concerns estimators supported on a single
spanning tree of the terminal clique, and so applies directly to the
base $\rho=1$ update. For $\rho>1$, \CAST-$\rho$ samples one spanning
tree on the expanded Schur clique and contracts the auxiliary-copy
blocks back to the terminals; the resulting terminal update need not
itself be a tree, and is therefore not a competing tree distribution
on $K_v$.

We analyze the general $\rho$ construction through a different local
certificate: the normalized second moment of the one-pivot Schur
error. The bound below scales as $1/\rho$, quantifying how splitting
decomposes the local estimator into smaller normalized contributions.

Let $E^{\mathrm{exp}}$ denote the edge set of the expanded Schur
clique. For each auxiliary edge $e=(p,q)\in E^{\mathrm{exp}}$, let
$H_e$ be its contracted terminal contribution, with $H_e=0$ when
$\phi(p)=\phi(q)$, and set $A_e=K_v^{+/2}H_eK_v^{+/2}$, its size in
the energy geometry of $K_v$. Let
$\Pi_{K_v}=K_v^{+/2}K_vK_v^{+/2}$ be the orthogonal projection onto
$\operatorname{range}(K_v)=\mathbf 1_d^\perp$.

\begin{lemma}[Normalized auxiliary-edge contribution]
\label{lem:atom-bound}
For every auxiliary edge $e\in E^{\mathrm{exp}}$,
$0\preceq A_e\preceq\frac{1}{\rho}\Pi_{K_v}$. Moreover,
$A_e^2=\frac{1}{\rho}A_e$.
\end{lemma}

\begin{proof}
See Appendix~\ref{app:atom_bound}.
\end{proof}

Each auxiliary edge therefore contributes at most $1/\rho$ in the
energy geometry of $K_v$, independently of the pivot degree and of the
incident weights: splitting shrinks the largest possible single
contribution in direct proportion to $\rho$. The second-moment
analysis combines this with the determinantal negative dependence of
weighted random spanning-tree edge indicators.

\begin{lemma}[Covariance domination for spanning-tree indicators]
\label{lem:dpp}
Let $T$ be a weighted random spanning tree of a connected weighted
graph, and for each edge $e$ let $X_e=\mathbf 1_{\{e\in T\}}$ and
$p_e=\mathbb{E}[X_e]$. If $C$ is the covariance matrix of $X$, then
$C\preceq\operatorname{diag}(p)$.
\end{lemma}

\begin{proof}
See Appendix~\ref{app:dpp}.
\end{proof}

Negative dependence is what makes the tree structure work in our
favor: the sampled edges are not independent, but their covariance
never exceeds what independent sampling with the same marginals would
give, so the second moment can be bounded edgewise.

\begin{theorem}[Local second-moment bound for
\texorpdfstring{\CAST-$\rho$}{CAST-rho}]
\label{thm:local-second-moment}
Let $Y_v=K_v^{+/2}(\widehat K_v-K_v)K_v^{+/2}$ be the normalized local
Schur error of the \CAST-$\rho$ update. For every integer $\rho\geq1$,
\[
\mathbb{E}[Y_v]=0,
\qquad
\mathbb{E}[Y_v^2]\preceq\frac{1}{\rho}\,\Pi_{K_v}.
\]
\end{theorem}

\begin{proof}
See Appendix~\ref{app:local-second-moment}.
\end{proof}

At $\rho=1$ the bound reads $\mathbb{E}[Y_v^2]\preceq\Pi_{K_v}$,
matching the unit local risk attained by \CASTone{}
(Theorem~\ref{thm:minimax}); each further doubling of $\rho$ halves
the certified bound.

\paragraph{Scope of the guarantees.} The results above certify the \CAST{} primitive at a single elimination step: for every $\rho$ the update is unbiased with connected support, at $\rho=1$ its marginals uniquely minimize the local risk, and for general $\rho$ its normalized second moment obeys a $1/\rho$ bound. The second moment is the quantity that matrix-concentration analyses of approximate elimination control alongside the largest single increment, so the bound identifies $\rho$ as a principled control on sampling variability, uniform over pivot degrees and incident weights. It remains a local certificate: it does not track how these errors accumulate across eliminations, and it says nothing about the magnitude of the resulting gain or about the construction cost and downstream fill that splitting introduces. Section~\ref{sec:empirical_methods} measures these effects across four benchmark families.


\section{Experimental Evaluation}
\label{sec:empirical_methods}


\paragraph{Methods.}
We compare four pre\-conditioners within a common elimination and PCG
framework. \AC~\cite{AC} is the practical approximate Cholesky
factorization based on the elimination estimator of Kyng and
Sachdeva~\cite{kyng2016approximate}, implemented in
\texttt{Laplacians.jl}\footnote{\url{https://github.com/danspielman/Laplacians.jl}};
at each pivot it replaces the exact Schur-complement clique by a
connected update generated through randomized sequential edge pairing.
\ACtwo{} is its doubled-budget variant (\texttt{split}~$=$~\texttt{merge}~$=2$):
each edge is represented by up to two half-weight multiedges, with fill
multiplicity capped at two. These are the state of the art on the
families our corpus is drawn from, where multigrid and
incomplete-Cholesky solvers~\cite{cmg,hypre,petsc,livne2012lean,icc}
each fail on some instance while \AC{} and \ACtwo{} converge
throughout~\cite{AC}, as our own runs on these collections confirmed; we therefore confine
the comparison to the randomized approximate-elimination family.

\CASTone{} applies the base $\rho=1$ update, replacing each elimination
clique by a weighted random spanning tree whose edge marginals equal
the clique leverage scores (Section~\ref{sec:cast-update}). \CASTtwo{}
applies two-way splitting: each pivot neighbor is temporarily replaced
by two equal-conductance auxiliary copies, one weighted random spanning
tree is sampled from the resulting expanded Schur clique, and the
copies are contracted back to the original neighborhood. We compare
\AC{} with \CASTone{} as the base-granularity methods and \ACtwo{} with
\CASTtwo{} as their doubled-granularity variants; all four share the
same elimination ordering (greedy on approximate
minimum unweighted degree, following~\cite{AC}), factor representation, and PCG
implementation, so the comparison isolates the local update.

\paragraph{\ACtwo{} versus \CASTtwo{}.}
The two robustness variants increase sampling granularity in
different ways. In \ACtwo{}, edge replication persists throughout the
factorization: the residual operator is represented as a multigraph
with at most two multiedges per vertex pair, and a current neighbor
represented by $t\in\{1,2\}$ multiedges contributes $t$ samples to the
sequential clique update. Thus \ACtwo{} is not equivalent to
averaging two independent \AC{} factorizations.

In \CASTtwo{}, splitting is local to the current pivot. Eliminating a
degree-$d$ pivot creates $2d$ temporary auxiliary copies, one weighted
random spanning tree is sampled on the expanded Schur clique, and its
copy blocks are contracted onto the $d$ original neighbors. The
auxiliary copies are then discarded, although the contracted terminal
edges remain in the residual graph and may increase downstream fill.
Hence \ACtwo{} and \CASTtwo{} are comparable doubled-granularity
variants, but they differ in sample dependence, persistence, and the
distribution of fill.

\paragraph{Metrics.}
For $q=250$ right-hand sides, we report the total reuse workload
\[
T_{250} = T_{\mathrm{build}} + \sum_{r=1}^{250}T_{\mathrm{solve},r}.
\]
The two primary paired speedups are
\[
S_1 = \frac{T_{250}(\AC)}{T_{250}(\CASTone)},
\qquad
S_2 = \frac{T_{250}(\ACtwo)}{T_{250}(\CASTtwo)},
\]
where values greater than one favor the corresponding \CAST{}
variant. Aggregate speedups are \emph{arithmetic} means of per-system
ratios, computed in the direction stated for each comparison; we also
report the number of systems on which each \CAST{} variant is faster,
which is independent of this convention.
Following~\cite{AC}, we also report factor-construction and
per-right-hand-side solve costs normalized by the number of input
non\-zeros,
\[
C_{\mathrm{build}} = 10^6\,
\frac{T_{\mathrm{build}}}{\operatorname{nnz}(A)},
\qquad
C_{\mathrm{solve}} = 10^6\,
\frac{\overline T_{\mathrm{solve}}}{\operatorname{nnz}(A)},
\]
with $\overline T_{\mathrm{solve}}$ the average per-solve time of
Section~\ref{sec:preliminaries}. Both normalized costs are reported in
$\mu\mathrm{s}/\mathrm{nnz}$.

\paragraph{Setup.}
Unless otherwise stated, every factorization is evaluated only as a
pre\-conditioner for the same PCG implementation, which uses a
recurrence-based relative-residual stopping threshold of $10^{-8}$ and
an iteration cap of $10^3$. After each solve, we explicitly compute
$\|A\mathbf x-\mathbf b\|_2/\|\mathbf b\|_2$ to verify the requested
tolerance. We report any discrepancy between the recurrence-based
stopping test and this explicit residual check; the treatment of
failed solves in each aggregate is stated alongside the corresponding
result.

For each matrix and method, one factor is constructed and reused for
$q=250$ right-hand sides. The right-hand sides are independent
standard Gaussian vectors projected onto $\mathbf 1^{\perp}$, as
required for Laplacian compatibility. A fixed right-hand-side
sequence, generated independently of the factor-sampling seeds, is
presented to every method and every factor draw, so comparisons are
paired at the level of individual right-hand sides. The number of
independently seeded factor draws and the aggregation rule are stated
for each collection in the corresponding appendix; all methods
entering a paired comparison use the same right-hand-side sequence and
the same number of factor draws.

\paragraph{Matrix collections.}
The evaluation corpus contains $201$ SDDM and Laplacian systems from
the SDDM2023 benchmark
suite\footnote{\url{https://rjkyng.github.io/SDDM2023/}} of Gao, Kyng,
and Spielman: $28$ Suite\-Sparse matrices, $128$ Chimera-IPM systems,
$40$ Spielman-IPM systems, and five Sachdeva-star instances. Eleven
Suite\-Sparse matrices are excluded from the comparative aggregates
as described below, leaving $190$ systems with paired measurements.
Collection-specific protocols and per-system results appear in the
appendix. All experiments run in a single process on an Apple Silicon MacBook
Pro with $24$\,GB of RAM under Julia~1.12. Pre\-conditioner
construction and application are single-threaded, and dense operations
in the shared PCG implementation use the same OpenBLAS configuration
for all methods.

The reported Spielman corpus contains the four scales
$k\in\{100,200,300,400\}$. The larger $k\in\{500,600\}$ sequences,
containing approximately $1.5\times10^8$ and $2.2\times10^8$
non\-zeros, exceed the memory capacity of the benchmark machine under
the multi-solve protocol and are not included.

Within the Suite\-Sparse collection, ten matrices are retained as
correctness checks but excluded from comparative aggregates because
all four methods solve them in one PCG iteration and their
timings are dominated by fixed construction and application overhead. We also
exclude \texttt{bcsstm25}: at the requested tolerance of $10^{-8}$,
all methods reach the iteration cap for at least some right-hand
sides, whereas at tolerance $10^{-5}$ every method converges in one
iteration, indicating that the failure is associated with the
matrix--tolerance pair rather than a particular pre\-conditioner. The
Suite\-Sparse comparative aggregate therefore contains the remaining
$17$ matrices.

\subsection{Cross-collection results}
\label{sec:cross-collection}

Table~\ref{tab:cross_collection} summarizes the paired comparative
subsets at $q=250$ right-hand sides per factor. Across the $190$
systems with paired measurements, selecting the faster of \CASTone{}
and \CASTtwo{} for each system in hindsight gives an arithmetic-mean
speedup of $1.9\times$ over the correspondingly faster of \AC{} and
\ACtwo{}, with \CAST{} faster on $178$ systems. Per-system results 
 are in Appendix ~\ref{app:additional_exp}.

\paragraph{Base granularity.}
\CASTone{} improves on \AC{} across every collection, and does so
uniformly rather than on average: it is faster on all $128$
Chimera-IPM systems, with per-system ratios between $1.08\times$ and
$1.31\times$ and a mean of $1.18\times$, and on $14$ of $17$
Suite\-Sparse matrices with a mean of $1.17\times$. 
Against a tuned implementation that is the most reliable available solver on these families, a consistent $1.17$--$1.18\times$ speedup
in total reuse workload---concentrated in the solve phase, and
therefore compounding with the number of right-hand sides---is a
substantive gain from changing only the local sampling distribution. On the Sachdeva-star ladder the margin instead grows with instance
size, from $1.58\times$ at $k=50$ to $5.30\times$ at $k=600$
($3.09\times$ on average), as \AC{}'s iteration count degrades faster
than \CASTone{}'s.

\paragraph{Doubled granularity.}
The $\rho=2$ comparison is bimodal, and the elimination-degree profile
predicts which regime holds. When pivot degrees are uniformly small,
splitting is nearly free and the gain is large: on the Spielman-IPM
systems, where every pivot has degree at most five, \CASTtwo{} reaches
the verified $10^{-8}$ tolerance in one to two PCG iterations and is
$4.43\times$ faster than \ACtwo{}. The Sachdeva-star family benefits
through a different mechanism. Its neighborhoods are clique-dominated
rather than low-degree and \CASTtwo{} needs multiple iterations for
convergence, but it holds iteration growth to $25$--$45$ across
the ladder against $28$--$65$ for \ACtwo{}. When
elimination produces a \emph{heavy tail} of high-degree pivots the additional
contributions instead dominate: on the Chimera collection \ACtwo{} is
faster on $101$ of $128$ systems, \CASTtwo{} is faster only on instance
$i_3$, whose maximum pivot degree of $48$ is the smallest in the
family. On Suite\-Sparse both are close, \CASTtwo{} leading by
$1.07\times$ on $11$ of $17$ matrices.

\begin{table*}[t]
\centering
\caption{Cross-collection performance at $q=250$ right-hand sides per
factor. Entries are per-collection aggregate speedups; the aggregation rule
is stated in the corresponding Appendix~\ref{app:additional_exp}. Parentheses give the number of systems on which the
\CAST variant is faster. The final
column is a post hoc  comparison that selects the faster
configuration within each method family separately for each system:
$\min(T_{\AC},T_{\ACtwo})/
 \min(T_{\CASTone},T_{\CASTtwo})$.}
\label{tab:cross_collection}
\begin{tabular}{l l c c c}
\toprule
Collection
& Elimination profile
& $\frac{\AC}{\CASTone}$
& $\frac{\ACtwo}{\CASTtwo}$
& $\frac{\min(\AC,\ACtwo)}
        {\min(\CASTone,\CASTtwo)}$ \\
\midrule
Suite\-Sparse 
& mostly low-degree, $\bar d\lesssim7$

& $1.17\times$ (14/17)
& $1.07\times$ (11/17)
& $1.11\times$ (12/17) \\

Chimera-IPM 
& $\bar d=6$--$27$, $d_{\max}=48$--$363$

& $1.18\times$ (128/128)
& $0.65\times$ (27/128)
& $1.16\times$ (121/128) \\

Spielman-IPM 
& all pivots satisfy $d\leq5$

& $1.27\times$ (38/40)
& $4.43\times$ (40/40)
& $4.43\times$ (40/40) \\

Sachdeva-star 
& clique-dominated

& $3.09\times$ (5/5)
& $1.90\times$ (5/5)
& $1.90\times$ (5/5) \\
\bottomrule
\end{tabular}
\end{table*}

\paragraph{\textbf{Choosing $\rho$}.} Two conclusions follow. First, the sampling distribution matters on its
own. All four methods share an elimination ordering, factor format, and
PCG path; \AC{} and \CASTone{} differ only in how the local update is
drawn. That difference alone cuts total reuse workload by $14$--$15\%$
on the heterogeneous collections, almost entirely in the solve phase.

Second, $\rho$ is a structural choice, not a strictly better setting.
Splitting pays when its extra terminal-edge contributions stay
confined: under small pivot degrees, as on Spielman-IPM, or in
clique-dominated neighborhoods where iterations would otherwise grow,
as on Sachdeva-star. It costs when a heavy degree tail carries them
into later eliminations, compounding as downstream fill. Degree alone
does not decide this; the clique-dominated case shows why. What matters
is whether the contributions propagate. A \CASTone{} elimination
profile predicts this on our corpus, but reading it requires a first
factorization, and we do not evaluate it as an adaptive rule. The
post hoc column of Table~\ref{tab:cross_collection} therefore bounds what
any per-system choice between $\rho=1$ and $\rho=2$ could achieve.

\section{Related Work}
\label{sec:related-work}

\paragraph{Laplacian and SDDM solvers.}
Nearly-linear-time Laplacian solvers originated from support-graph
preconditioning, low-stretch graph constructions, and spectral
sparsification~\cite{spielman2004nearly,koutis2011nearly,
cohen2014solving,jambulapati2021ultrasparse}, with
effective-resistance sampling an especially influential route to the
last~\cite{spielman2008graph}. These works give strong global
approximation and running-time guarantees, but their recursive
preconditioning structures differ from the sparse approximate
factorizations studied here. Practical solvers for these systems also
include incomplete Cholesky and algebraic multigrid, among them
graph-specialized variants such as \textsf{LAMG}~\cite{livne2012lean}.
\CAST{} belongs instead to the randomized approximate-elimination
line, which builds a sparse factor by eliminating vertices
sequentially and replacing each dense Schur-complement clique by a
sparse random update.

\paragraph{Randomized approximate elimination.}
Kyng and Sachdeva introduced a nearly-linear-time approximate Gaussian
elimination algorithm for graph Laplacians, replacing elimination fill
by unbiased random samples and controlling the accumulated error
through a matrix-martingale analysis~\cite{kyng2016approximate}.
Related sparsified-Cholesky constructions extend the approach to
connection Laplacians~\cite{kyng2016sparsified}. On the practical
side, \textsf{RCHOL} adds shared-memory
parallelization~\cite{chen2021rchol}, and more recent work studies
parallel frameworks and CPU/GPU construction of randomized
approximate-Cholesky
pre\-conditioners~\cite{baumann2024framework,liang2025parallel}. These
address global guarantees, alternative sampling schemes, or parallel
implementation. \CAST{} instead studies the distribution used for a
single local Schur-clique replacement, then applies that primitive
within a sequential factorization.

\paragraph{Practical approximate-Cholesky solvers.}
The closest predecessor is the AC($k$) framework of Gao, Kyng, and
Spielman~\cite{AC}, which turns randomized approximate elimination
into a practical SDDM solver in \texttt{Laplacians.jl}. Its local
updates preserve connected support, and its configurations AC and AC2
use one and two samples per entry to trade construction cost against
robustness. CAST adopts the same connected-update principle but
differs in the distribution. AC($k$) generates its trees by randomized
sequential edge pairing, so the induced edge marginals depend on the
order in which incident edges are processed. CAST samples from the
weighted random spanning-tree distribution of the clique itself, which
is order-independent and attains the minimax edge marginals of
Theorem~\ref{thm:minimax}.

\paragraph{Leverage scores and random spanning trees.}
Edge leverage scores---equivalently, edge conductance times effective
resistance---are central to spectral
sparsification~\cite{spielman2008graph}. For a weighted random
spanning tree, the transfer-current theorem identifies the inclusion
probability of each edge with its leverage
score~\cite{lyons2003determinantal}, and random spanning trees have
accordingly been studied as spectral sparsifiers, including from
unions of a few independent trees~\cite{kyng2018matrix}. Sampling such
trees on general graphs requires nontrivial graph-algorithmic
machinery~\cite{durfee2017sampling}.

\CAST{} uses this distribution locally rather than globally, which is
what makes it cheap. Sampling one tree per elimination clique avoids
the general machinery entirely: the star origin of the clique gives
its leverage scores in closed form, and the product-form conductances
of the expanded clique admit exact sampling in $\bigO(\rho d)$ time
via weighted Pr\"ufer codes~\cite{aigner1999proofs,west2001introduction}.
The union-of-trees results of~\cite{kyng2018matrix} are also not the
right analogy for $\rho>1$: \CAST-$\rho$ samples a single tree on an
expanded vertex set, not $\rho$ independent trees on the original terminals.

\paragraph{\textbf{Conclusion}}
\label{sec:conclusion}
We introduced \CAST{}, a canonical sparse replacement for the Schur
clique created during approximate Cholesky elimination. We show that among unbiased
inverse-marginal one-tree estimators, leverage-score marginals
uniquely minimize the largest normalized sampled-edge contribution.
Because the clique is induced by a star, \CAST{} samples the
corresponding weighted random spanning tree exactly in
$\bigO(\rho d)$ time. Empirically, \CASTone{} reduces total reuse
workload by $14$--$15\%$ relative to \AC{} on SuiteSparse and
Chimera-IPM systems. When additional fill remains controlled, \CASTtwo{} improves
robustness: on Spielman-IPM systems it reaches the verified $10^{-8}$
residual tolerance in one to two PCG iterations and is, on average,
$4.43\times$ faster than \ACtwo{}. 

\bibliographystyle{plain}    
\bibliography{references}

@String{Computing = "Computing" }

@String{Computer = "{IEEE} Computer" }

@String{Springer = "Springer-Verlag" }

@inproceedings{kyng2016approximate,
  title={Approximate gaussian elimination for laplacians-fast, sparse, and simple},
  author={Kyng, Rasmus and Sachdeva, Sushant},
  booktitle={2016 IEEE 57th Annual Symposium on Foundations of Computer Science (FOCS)},
  pages={573--582},
  year={2016},
  organization={IEEE}
}

@inproceedings{page1999pagerank,
  title={The pagerank citation ranking: Bring order to the web},
  author={Page, Lawrence and Brin, Sergey and Motwani, Rajeev and Winograd, Terry},
  booktitle={Proc. of the 7th International World Wide Web Conf.--1998},
  year={1999}
}

@inproceedings{zhu2003semi,
  title={Semi-supervised learning using gaussian fields and harmonic functions},
  author={Zhu, Xiaojin and Ghahramani, Zoubin and Lafferty, John D},
  booktitle={Proceedings of the 20th International conference on Machine learning (ICML-03)},
  pages={912--919},
  year={2003}
}

@article{huang2019isira,
  title={iSIRA: Integrated shift--invert residual Arnoldi method for graph Laplacian matrices from big data},
  author={Huang, Wei-Qiang and Lin, Wen-Wei and Lu, Henry Horng-Shing and Yau, Shing-Tung},
  journal={Journal of Computational and Applied Mathematics},
  volume={346},
  pages={518--531},
  year={2019},
  publisher={Elsevier}
}

@article{gleich2015pagerank,
  title={PageRank beyond the web},
  author={Gleich, David F},
  journal={siam REVIEW},
  volume={57},
  number={3},
  pages={321--363},
  year={2015},
  publisher={SIAM}
}

@phdthesis{gremban1996combinatorial,
  title={Combinatorial preconditioners for sparse, symmetric, diagonally dominant linear systems},
  author={Gremban, Keith D},
  year={1996},
  school={Carnegie Mellon University Pittsburgh}
}

@book{west2001introduction,
  author    = {West, Douglas B.},
  title     = {Introduction to Graph Theory},
  edition   = {2nd},
  publisher = {Prentice Hall},
  year      = {2001}
}

@book{aigner1999proofs,
  title={Proofs from the Book},
  author={Aigner, Martin and Ziegler, G{\"u}nter M},
  journal={Germany},
  volume={1},
  number={2},
  pages={7},
  year={1999},
  publisher={Springer}
}

@article{cmg,
  title={Combinatorial preconditioners and multilevel solvers for problems in computer vision and image processing},
  author={Koutis, Ioannis and Miller, Gary L and Tolliver, David},
  journal={Computer Vision and Image Understanding},
  volume={115},
  number={12},
  pages={1638--1646},
  year={2011},
  publisher={Elsevier}
}

@inproceedings{hypre,
  title={hypre: A library of high performance preconditioners},
  author={Falgout, Robert D and Yang, Ulrike Meier},
  booktitle={International Conference on computational science},
  pages={632--641},
  year={2002},
  organization={Springer}
}

@techreport{petsc,
  title       = {{PETSc} Users Manual},
  author      = {Balay, Satish and Abhyankar, Shrirang and Adams, Mark and
                 Brown, Jed and Brune, Peter and Buschelman, Kris and
                 Dalcin, Lisandro and Dener, Alp and Eijkhout, Victor and
                 Gropp, William and others},
  institution = {Argonne National Laboratory},
  year        = {2019}
}

@article{icc,
  title   = {An iterative solution method for linear systems of which the
             coefficient matrix is a symmetric {M}-matrix},
  author  = {Meijerink, J. A. and van der Vorst, Henk A.},
  journal = {Mathematics of Computation},
  volume  = {31},
  number  = {137},
  pages   = {148--162},
  year    = {1977}
}

@unpublished{baumann2026vac,
  title  = {{VAC}: A Volume-Sampling-Based Elimination Rule for Approximate Cholesky Factorization},
  author = {Baumann, Yves and Kyng, Rasmus and Z{\"o}cklein, Gernot},
  year   = {2026},
  month  = jul,
  note   = {Manuscript, July 28, 2026}
}

@article{zhou2003learning,
  title={Learning with local and global consistency},
  author={Zhou, Dengyong and Bousquet, Olivier and Lal, Thomas and Weston, Jason and Sch{\"o}lkopf, Bernhard},
  journal={Advances in neural information processing systems},
  volume={16},
  year={2003}
}

@inproceedings{koutis2011nearly,
  title={A nearly-m log n time solver for sdd linear systems},
  author={Koutis, Ioannis and Miller, Gary L and Peng, Richard},
  booktitle={2011 IEEE 52nd Annual Symposium on Foundations of Computer Science},
  pages={590--598},
  year={2011},
  organization={IEEE}
}

@inproceedings{cohen2014solving,
  title={Solving SDD linear systems in nearly m log1/2 n time},
  author={Cohen, Michael B and Kyng, Rasmus and Miller, Gary L and Pachocki, Jakub W and Peng, Richard and Rao, Anup B and Xu, Shen Chen},
  booktitle={Proceedings of the forty-sixth annual ACM symposium on Theory of computing},
  pages={343--352},
  year={2014}
}

@article{jambulapati2021ultrasparse,
  title={Ultrasparse ultrasparsifiers and faster laplacian system solvers},
  author={Jambulapati, Arun and Sidford, Aaron},
  journal={ACM Transactions on Algorithms},
  volume={21},
  number={3},
  pages={1--49},
  year={2025},
  publisher={ACM New York, NY}
}

@inproceedings{andersen2006local,
  title={Local graph partitioning using pagerank vectors},
  author={Andersen, Reid and Chung, Fan and Lang, Kevin},
  booktitle={2006 47th annual IEEE symposium on foundations of computer science (FOCS'06)},
  pages={475--486},
  year={2006},
  organization={IEEE}
}

@inproceedings{spielman2004nearly,
  title={Nearly-linear time algorithms for graph partitioning, graph sparsification, and solving linear systems},
  author={Spielman, Daniel A and Teng, Shang-Hua},
  booktitle={Proceedings of the thirty-sixth annual ACM symposium on Theory of computing},
  pages={81--90},
  year={2004}
}

@inproceedings{spielman2008graph,
  title={Graph sparsification by effective resistances},
  author={Spielman, Daniel A and Srivastava, Nikhil},
  booktitle={Proceedings of the fortieth annual ACM symposium on Theory of computing},
  pages={563--568},
  year={2008}
}

@article{livne2012lean,
  title={Lean algebraic multigrid (LAMG): Fast graph Laplacian linear solver},
  author={Livne, Oren E and Brandt, Achi},
  journal={SIAM Journal on Scientific Computing},
  volume={34},
  number={4},
  pages={B499--B522},
  year={2012},
  publisher={SIAM}
}

@inproceedings{kyng2016sparsified,
  title={Sparsified cholesky and multigrid solvers for connection laplacians},
  author={Kyng, Rasmus and Lee, Yin Tat and Peng, Richard and Sachdeva, Sushant and Spielman, Daniel A},
  booktitle={Proceedings of the forty-eighth annual ACM symposium on Theory of Computing},
  pages={842--850},
  year={2016}
}

@article{chen2021rchol,
  title={RCHOL: Randomized Cholesky factorization for solving SDD linear systems},
  author={Chen, Chao and Liang, Tianyu and Biros, George},
  journal={SIAM Journal on Scientific Computing},
  volume={43},
  number={6},
  pages={C411--C438},
  year={2021},
  publisher={SIAM}
}

@inproceedings{baumann2024framework,
  title={A Framework for Parallelizing Approximate Gaussian Elimination},
  author={Baumann, Yves and Kyng, Rasmus},
  booktitle={Proceedings of the 36th ACM Symposium on Parallelism in Algorithms and Architectures},
  pages={195--206},
  year={2024}
}

@article{liang2025parallel,
  title={Parallel GPU-Accelerated Randomized Construction of Approximate Cholesky Preconditioners},
  author={Liang, Tianyu and Chen, Chao and Yaniv, Yotam and Luo, Hengrui and Tench, David and Li, Xiaoye S and Buluc, Aydin and Demmel, James},
  journal={arXiv preprint arXiv:2505.02977},
  year={2025}
}

@article{AC,
  title={AC(k): Robust Solution of Laplacian Equations by Randomized Approximate Cholesky Factorization},
  author={Gao, Yuan and Kyng, Rasmus and Spielman, Daniel A},
  journal={SIAM Journal on Scientific Computing },
  year={2026}
}

@article{lyons2003determinantal,
  title={Determinantal probability measures},
  author={Lyons, Russell},
  journal={Publications Math{\'e}matiques de l'IH{\'E}S},
  volume={98},
  pages={167--212},
  year={2003}
}

@inproceedings{kyng2018matrix,
  title={A matrix chernoff bound for strongly rayleigh distributions and spectral sparsifiers from a few random spanning trees},
  author={Kyng, Rasmus and Song, Zhao},
  booktitle={2018 IEEE 59th Annual Symposium on Foundations of Computer Science (FOCS)},
  pages={373--384},
  year={2018},
  organization={IEEE}
}

@inproceedings{durfee2017sampling,
  title={Sampling random spanning trees faster than matrix multiplication},
  author={Durfee, David and Kyng, Rasmus and Peebles, John and Rao, Anup B and Sachdeva, Sushant},
  booktitle={Proceedings of the 49th Annual ACM SIGACT Symposium on Theory of Computing},
  pages={730--742},
  year={2017}
}

\appendix

\section{Background on Pr\"ufer Codes}
\label{app:prufer-background}

This appendix collects the facts about Pr\"ufer codes underlying the
exact sampler of Section~\ref{sec:prufer}.

\paragraph{The Pr\"ufer correspondence.}
Let $P$ be a set of $m\geq 2$ labeled vertices with an arbitrary total
ordering. A \emph{Pr\"ufer code} is a sequence
$S=(S_1,\ldots,S_{m-2})\in P^{m-2}$. The Pr\"ufer correspondence is a
bijection between such sequences and labeled spanning trees on
$P$~\cite{aigner1999proofs,west2001introduction}: every labeled tree
has a unique code, and every sequence of length $m-2$ decodes to a
unique tree. To encode a tree, repeatedly remove the smallest-labeled
leaf and append its unique neighbor to the sequence, stopping when two
vertices remain; one vertex is removed per step, so the code has
exactly $m-2$ symbols.

For example, the tree on $\{1,2,3,4,5\}$ with edges
$\{(1,3),(2,3) \\,(3,4),(4,5)\}$ has smallest leaf $1$, whose neighbor is
$3$, so the first symbol is $3$. After removing $1$ the smallest leaf
is $2$, again with neighbor $3$; after removing $2$ it is $3$, whose
remaining neighbor is $4$. The code is $(3,3,4)$.

The inverse procedure reconstructs the tree. For each label $p$
initialize $r_p=1+\lvert\{t:S_t=p\}\rvert$. At each step choose the
smallest label $\ell$ with $r_\ell=1$, connect $\ell$ to the next code
symbol $S_t$, remove $\ell$, and decrease $r_{S_t}$ by one. Two labels
remain once all symbols are processed; connecting them completes the
tree. When $m=2$ the code is empty and the tree is the single edge
between the two labels.

\paragraph{Degrees from code multiplicities.}
The property \CAST{} relies on is that $p$ appears in the code exactly
$\deg_T(p)-1$ times, i.e.
\[
\deg_T(p) = 1+\bigl\lvert\{t:S_t=p\}\bigr\rvert .
\]
Intuitively, $p$ is recorded whenever a neighboring leaf is removed
while $p$ remains, which happens once per incident edge except the
edge through which $p$ is itself removed, or the final edge if $p$ is
one of the last two vertices. Consistently,
$\sum_{p\in P}(\deg_T(p)-1)=2(m-1)-m=m-2$, the length of the code.

\paragraph{Nonuniform Pr\"ufer sampling.}
The correspondence is purely combinatorial. A weighted tree
distribution can nonetheless be induced by drawing the code symbols
nonuniformly: if $S_1,\ldots,S_{m-2}$ are independent with
$\Pr{S_t=p}=\pi_p$, then, because the correspondence is a bijection,
the decoded tree $T$ has probability
$\Pr{T}=\prod_{p\in P}\pi_p^{\deg_T(p)-1}$, which depends on $T$ only
through its degree sequence. This matches the structure of complete
graphs whose edge weights factor over endpoints.

\begin{lemma}[Weighted Pr\"ufer sampling]
\label{lem:prufer}
Let the complete graph on $P$ have edge weights
$w_{pq}=\beta\,\theta_p\theta_q$ for $p\neq q$, with $\beta>0$ and
$\theta_p>0$. Draw $S_1,\ldots,S_{m-2}$ independently with
$\Pr{S_t=p}=\theta_p/\sum_{r\in P}\theta_r$ and let $T$ be the decoded
tree. Then $T$ follows the weighted random spanning-tree distribution,
$\Pr{T}\propto\prod_{(p,q)\in T}w_{pq}$.
\end{lemma}

\begin{proof}
A spanning tree has $m-1$ edges and each $p$ appears in $\deg_T(p)$ of
them, so
$\prod_{(p,q)\in T}w_{pq}
 =\beta^{m-1}\prod_{p\in P}\theta_p^{\deg_T(p)}$.
By the degree-multiplicity property,
\[
\Pr{T}=\prod_{p\in P}
\left(\frac{\theta_p}{\sum_{r}\theta_r}\right)^{\deg_T(p)-1}
=\Bigl(\sum_{r}\theta_r\Bigr)^{-(m-2)}
\prod_{p\in P}\theta_p^{\deg_T(p)-1}.
\]
The ratio of the two expressions is
$\beta^{m-1}\bigl(\sum_r\theta_r\bigr)^{m-2}\prod_{p}\theta_p$, which
does not depend on $T$. Both are distributions over the same finite
set, so they coincide.
\end{proof}

The endpoint-product form is essential: independent Pr\"ufer symbols
do not sample the correct weighted spanning-tree distribution for
general edge weights.

\paragraph{Application to \CAST-$\rho$.}
At a pivot with $d$ terminals, the expanded Schur clique is the
complete graph on the $m=\rho d$ auxiliary copies with conductances
$w^{\mathrm{exp}}_{pq}=c_pc_q/a$. This is
Lemma~\ref{lem:prufer} with $\theta_p=c_p$ and $\beta=1/a$, and since
$\sum_{p\in P}c_p=a$, the symbol distribution is
$\Pr{S_t=p}=c_p/a$. Because every copy of terminal $i$ carries the
same conductance $a_i/\rho$, a symbol can be drawn hierarchically:
sample $i$ with probability $a_i/a$, then one of its $\rho$ copies
uniformly. Section~\ref{sec:prufer} gives the resulting sampler and
its $\bigO(\rho d)$ cost.

\section{Proofs}

\subsection{Proof of Proposition~\ref{prop:cost}}
\label{app:prop_cost}
\begin{proof}
Let $m=\rho d$. A spanning tree $T$ of the expanded Schur clique has
$m-1$ edges, and each copy $p$ appears in $\deg_T(p)$ of them, so
\begin{equation*}
\prod_{(p,q)\in T} w_{pq}^{\mathrm{exp}}
= a^{-(m-1)} \prod_{p\in P} c_p^{\deg_T(p)} .
\end{equation*}
The Pr\"ufer correspondence is a bijection between spanning trees of
the complete graph on $P$ and sequences in $P^{m-2}$, under which $p$
appears exactly $\deg_T(p)-1$ times in the code of $T$. Hence drawing
the $m-2$ symbols independently with $\Pr{\text{symbol}=p}=c_p/a$
generates $T$ with probability
\begin{equation*}
\prod_{p\in P} \left(\frac{c_p}{a}\right)^{\deg_T(p)-1}
= a^{-(m-2)} \prod_{p\in P}c_p^{\deg_T(p)-1},
\end{equation*}
where we used $\sum_{p\in P}\bigl(\deg_T(p)-1\bigr)=m-2$. The ratio of
this probability to the unnormalized tree weight above is
$a\prod_{p\in P}c_p^{-1}$, independent of $T$. Two probability
distributions on the same finite set whose ratio is constant are
equal, so the sampler is exact.

For the running time, all copies of terminal $i$ carry the same
conductance $a_i/\rho$, so a symbol is drawn by sampling a terminal
from an alias table for $(a_1/a,\ldots,a_d/a)$ and then one of its
$\rho$ copies uniformly. The table is built once in $\bigO(d)$ time,
after which each of the $m-2$ symbols costs $\bigO(1)$. Decoding takes
$\bigO(m)$ time by the standard leaf-pointer algorithm. Contraction is
one pass over the $m-1$ auxiliary edges with $\bigO(1)$ work each: a
cross-block edge $(p,q)$ is emitted as the terminal-edge contribution
$(\phi(p),\phi(q),c_pc_q/(c_p+c_q))$, and a within-block edge is
discarded. The total is $\bigO(d+m)=\bigO(\rho d)$.

Finally, the only structures allocated are the alias table over the
$d$ terminals, the code of length $m-2$, and the decoded tree on $m$
copies. Neither the $\Theta(d^2)$-edge Schur clique nor the expanded
clique is ever formed.
\end{proof}

\subsection{Proof of Lemma~\ref{lem:aux-marginals}}
\label{app:aux_marginals}
\begin{proof}
Let $\mathrm{S}$ be the star on $P\cup\{z\}$ in which the center $z$
is joined to each auxiliary copy $r\in P$ by an edge of conductance
$c_r$, so that the total conductance at $z$ is
$\sum_{r\in P}c_r=a$. Eliminating $z$ from $\mathrm{S}$ produces, by
Eq.~\eqref{eqn:pivot_clique}, the clique on $P$ with edge conductances
$c_pc_q/a=w^{\mathrm{exp}}_{pq}$. The expanded Schur clique is
therefore the Schur complement of $\mathrm{S}$ onto $P$.

Schur complementation preserves effective resistances among the
retained vertices, so the effective resistance between $p$ and $q$ in
the expanded clique equals their effective resistance in
$\mathrm{S}$. Since $\mathrm{S}$ is a tree, the unique $p$--$q$ path
$p$--$z$--$q$ determines this resistance, giving
\[
R_{pq}^{\mathrm{exp}} = \frac{1}{c_p}+\frac{1}{c_q}.
\]
By the transfer-current theorem~\cite{lyons2003determinantal}, the
inclusion probability of an edge in a weighted random spanning tree is
its conductance times its effective resistance, whence
\[
\Pr{(p,q)\in T}
= w_{pq}^{\mathrm{exp}}R_{pq}^{\mathrm{exp}}
= \frac{c_pc_q}{a}\left(\frac{1}{c_p}+\frac{1}{c_q}\right)
= \frac{c_p+c_q}{a}. \qedhere
\]
\end{proof}

\subsection{Proof of Theorem~\ref{thm:cast-unbiased}}
\label{app:thm_unbiasedness}
\begin{proof}
Write $\widehat K_v=\sum_{e\in E^{\mathrm{exp}}}X_e H_e$ with
$X_e=\mathbf 1_{\{e\in T\}}$, so that by linearity
$\mathbb E[\widehat K_v]=\sum_{e}\Pr{e\in T}\,H_e$. Within-block edges
contract to self-loops and contribute $H_e=0$, so only pairs with
endpoints in distinct blocks remain.

Fix distinct terminals $i\neq j$. Collecting the terms carrying
$(\mathbf e_i-\mathbf e_j)(\mathbf e_i-\mathbf e_j)^\top$ and applying
Lemma~\ref{lem:aux-marginals},
\[
\begin{aligned}
\sum_{p\in P_i}\sum_{q\in P_j}\Pr{(p,q)\in T}\,\frac{c_pc_q}{c_p+c_q}
&= \sum_{p\in P_i}\sum_{q\in P_j}
   \frac{c_p+c_q}{a}\cdot\frac{c_pc_q}{c_p+c_q} \\
&= \frac{1}{a}\Bigl(\sum_{p\in P_i}c_p\Bigr)
            \Bigl(\sum_{q\in P_j}c_q\Bigr)
 = \frac{a_ia_j}{a},
\end{aligned}
\]
using $\sum_{p\in P_i}c_p=a_i$. This is exactly the coefficient of the
terminal edge $(i,j)$ in $K_v$. Both $\mathbb E[\widehat K_v]$ and
$K_v$ are graph Laplacians on $\{u_1,\ldots,u_d\}$, and a graph
Laplacian is determined by its edge coefficients, so 
$$
\mathbb{E}[\widehat K_v]=K_v.
$$

\end{proof}

\subsection{Proof of Theorem~\ref{thm:cast-connected}}
\label{app:thm_cast_connected}
\begin{proof}
The sampled tree $T$ is connected on the auxiliary-copy set $P$.
Contract each block $P_i$ to its terminal $u_i$. Contraction preserves
connectedness: given terminals $u_i$ and $u_j$, pick $p\in P_i$ and
$q\in P_j$; the $p$--$q$ path in $T$ maps to a walk from $u_i$ to
$u_j$ in the quotient graph.

It remains to check that the quotient edges are exactly the edges of
the support graph of $\widehat K_v$. An auxiliary edge with both
endpoints in one block becomes a self-loop and is discarded, which
does not affect connectivity. A cross-block edge $(p,q)$ with
$\phi(p)=i\neq j=\phi(q)$ contributes conductance
$c_pc_q/(c_p+c_q)>0$ to the terminal edge $(u_i,u_j)$, and parallel
contributions are summed. Since all $c_p>0$, every such contribution
is strictly positive and no cancellation occurs, so each quotient edge
carries positive weight in $\widehat K_v$. The support graph of
$\widehat K_v$ is therefore the quotient graph with self-loops
removed, and is connected.

When $\rho=1$ each block is a single copy, so the contraction is the
identity and $\widehat K_v$ is itself a spanning tree of the terminal
neighborhood; for $\rho>1$ it need not be a tree.
\end{proof}

\subsection{Proof of Corollary~\ref{cor:correct-rank}}
\label{app:correct-rank}
\begin{proof}
The nullspace of a weighted graph Laplacian has dimension equal to the
number of connected components of its support graph, and is spanned by
the indicator vectors of those components. By
Theorem~\ref{thm:cast-connected} the support graph of $\widehat K_v$
is connected on $\{u_1,\ldots,u_d\}$, so
$\ker(\widehat K_v)=\operatorname{span}\{\mathbf 1_d\}$ and, by
rank--nullity, $\operatorname{rank}(\widehat K_v)=d-1$.
\end{proof}

\subsection{Proof of Theorem~\ref{thm:minimax}}
\label{app:minimax}
\begin{proof}
Every spanning tree of the support graph of $K$ has exactly $d-1$
edges, so $\sum_{e\in E_K}\mathbf 1_{\{e\in T\}}=d-1$ for every
realization; taking expectations gives $\sum_{e\in E_K}p_e=d-1$. The
leverage scores satisfy $\sum_{e\in E_K}\tau_e=\operatorname{rank}(K)=d-1$
as well, so
\[
1=\frac{\sum_{e\in E_K}\tau_e}{\sum_{e\in E_K}p_e}
 =\frac{\sum_{e\in E_K}p_e\,(\tau_e/p_e)}{\sum_{e\in E_K}p_e}
 \leq \max_{e\in E_K}\frac{\tau_e}{p_e}=R_K(\mathcal D),
\]
since the middle expression is a weighted average of the ratios
$\tau_e/p_e$ with positive weights $p_e$. Hence
$R_K(\mathcal D)\geq 1$.

A weighted average with positive weights attains its maximum only when
every term equals that maximum. Equality therefore forces
$\tau_e/p_e=1$, i.e.\ $p_e=\tau_e$, for every $e\in E_K$; conversely,
if $p_e=\tau_e$ throughout then $R_K(\mathcal D)=1$. Finally, for the
weighted random spanning-tree distribution of $K$ the transfer-current
theorem~\cite{lyons2003determinantal} gives
$p_e=\Pr{e\in T}=w_e\,\mathbf b_e^\top K^{+}\mathbf b_e=\tau_e$, so
this distribution attains the bound.
\end{proof}

\subsection{Proof of Lemma~\ref{lem:atom-bound}}
\label{app:atom_bound}
\begin{proof}
If $e$ is a within-block edge then $H_e=0$, hence $A_e=0$ and both
claims hold trivially. Otherwise let $e=(p,q)$ with
$\phi(p)=i\neq j=\phi(q)$, and write
$\mathbf b=\mathbf e_i-\mathbf e_j$, so that
$H_e=\frac{c_pc_q}{c_p+c_q}\,\mathbf b\mathbf b^\top$ and
$A_e=\frac{c_pc_q}{c_p+c_q}\,(K_v^{+/2}\mathbf b)(K_v^{+/2}\mathbf b)^\top$.
Thus $A_e$ is positive semidefinite and rank one, and its only nonzero
eigenvalue is
\[
\|A_e\|
=\frac{c_pc_q}{c_p+c_q}\,\bigl\|K_v^{+/2}\mathbf b\bigr\|^2
=\frac{c_pc_q}{c_p+c_q}\,\mathbf b^\top K_v^{+}\mathbf b .
\]
By the argument of Lemma~\ref{lem:aux-marginals} applied to the pivot
star, the effective resistance between $u_i$ and $u_j$ in $K_v$ is
$\mathbf b^\top K_v^{+}\mathbf b=1/a_i+1/a_j$. Substituting
$c_p=a_i/\rho$ and $c_q=a_j/\rho$,
\[
\|A_e\|
=\frac{a_ia_j}{\rho\,(a_i+a_j)}
 \left(\frac{1}{a_i}+\frac{1}{a_j}\right)
=\frac{1}{\rho},
\]
independently of $d$ and of the incident weights.

Write $A_e=\frac{1}{\rho}\mathbf v\mathbf v^\top$ with
$\mathbf v=K_v^{+/2}\mathbf b/\|K_v^{+/2}\mathbf b\|$. Since
$\mathbf v\in\operatorname{range}(K_v^{+/2})=\operatorname{range}(K_v)$
and $\Pi_{K_v}$ is the orthogonal projection onto that subspace,
$\mathbf v\mathbf v^\top\preceq\Pi_{K_v}$, giving
$0\preceq A_e\preceq\frac{1}{\rho}\Pi_{K_v}$. Finally, a rank-one
matrix $\lambda\mathbf v\mathbf v^\top$ with $\|\mathbf v\|=1$
satisfies $(\lambda\mathbf v\mathbf v^\top)^2
=\lambda^2\mathbf v\mathbf v^\top=\lambda\cdot
\lambda\mathbf v\mathbf v^\top$, so $A_e^2=\frac{1}{\rho}A_e$.
\end{proof}

\subsection{Proof of Lemma~\ref{lem:dpp}}
\label{app:dpp}
\begin{proof}
Fix an arbitrary orientation of the edges. The edge indicators of a
weighted random spanning tree then form a determinantal point process
whose kernel is the transfer-current matrix $M$, which is symmetric
and positive semidefinite, with
$M_{ee}=p_e$~\cite{lyons2003determinantal}. Determinantal pair
correlations give, for $e\neq f$,
\[
\Pr{e\in T,\,f\in T}
=\det\begin{pmatrix} M_{ee} & M_{ef}\\ M_{fe} & M_{ff}\end{pmatrix}
=p_ep_f-M_{ef}^2 ,
\]
so $C_{ef}=\Pr{e\in T,f\in T}-p_ep_f=-M_{ef}^2$. On the diagonal,
$X_e$ is a Bernoulli indicator, so
$C_{ee}=\mathbb E[X_e^2]-p_e^2=p_e-p_e^2=p_e-M_{ee}^2$. Hence
\[
C=\operatorname{diag}(p)-M\circ M ,
\]
where $\circ$ is the entrywise product; note that $C$ does not depend
on the chosen orientation, since only the squares $M_{ef}^2$ appear.
Since $M\succeq 0$, the Schur product theorem gives
$M\circ M\succeq 0$, and therefore
$C\preceq\operatorname{diag}(p)$.
\end{proof}

\subsection{Proof of Theorem~\ref{thm:local-second-moment}}
\label{app:local-second-moment}
\begin{proof}
The mean-zero statement follows from
Theorem~\ref{thm:cast-unbiased}:
\[
\mathbb{E}[Y_v]
= K_v^{+/2}\bigl(\mathbb{E}[\widehat K_v]-K_v\bigr)K_v^{+/2}=0 .
\]
For $e\in E^{\mathrm{exp}}$ let $X_e=\mathbf 1_{\{e\in T\}}$ and
$p_e=\mathbb{E}[X_e]$. Since
$\widehat K_v=\sum_{e}X_eH_e$ and, by unbiasedness,
$K_v=\sum_{e}p_eH_e$, conjugating by $K_v^{+/2}$ gives
\[
Y_v=\sum_{e\in E^{\mathrm{exp}}}(X_e-p_e)A_e .
\]

Fix $\mathbf z\in\mathbb{R}^d$. As $Y_v$ is symmetric,
$\mathbf z^\top\mathbb{E}[Y_v^2]\mathbf z
=\mathbb{E}\bigl[\lVert Y_v\mathbf z\rVert_2^2\bigr]$, and expanding
the square,
\[
\mathbb{E}\bigl[\lVert Y_v\mathbf z\rVert_2^2\bigr]
=\sum_{e,f}\mathbb{E}\bigl[(X_e-p_e)(X_f-p_f)\bigr]
 \langle A_e\mathbf z,A_f\mathbf z\rangle
=\sum_{e,f}C_{ef}\,G_{ef},
\]
where $G_{ef}=\langle A_e\mathbf z,A_f\mathbf z\rangle$ is the Gram
matrix of the vectors $\{A_e\mathbf z\}$ and is therefore positive
semidefinite. The right-hand side is the Frobenius inner product
$\langle C,G\rangle_F$.

By Lemma~\ref{lem:dpp}, $\operatorname{diag}(p)-C\succeq 0$. For
positive-semidefinite $A$ and $B$ one has
$\langle A,B\rangle_F=\operatorname{tr}(AB)
=\operatorname{tr}(A^{1/2}BA^{1/2})\geq 0$, so
$\langle\operatorname{diag}(p)-C,\,G\rangle_F\geq 0$ and hence
\[
\mathbf z^\top\mathbb{E}[Y_v^2]\mathbf z
\;\leq\;\langle\operatorname{diag}(p),G\rangle_F
=\sum_{e}p_e\lVert A_e\mathbf z\rVert_2^2
=\mathbf z^\top\Bigl(\sum_{e}p_eA_e^2\Bigr)\mathbf z,
\]
the last equality because each $A_e$ is symmetric. As $\mathbf z$ was
arbitrary, $\mathbb{E}[Y_v^2]\preceq\sum_{e}p_eA_e^2$.

Lemma~\ref{lem:atom-bound} gives $A_e^2=\frac{1}{\rho}A_e$, so
$\sum_{e}p_eA_e^2=\frac{1}{\rho}\sum_{e}p_eA_e$, and unbiasedness
gives
\[
\sum_{e\in E^{\mathrm{exp}}}p_eA_e
= K_v^{+/2}\Bigl(\sum_{e}p_eH_e\Bigr)K_v^{+/2}
= K_v^{+/2}K_vK_v^{+/2}
= \Pi_{K_v}.
\]
Therefore $\mathbb{E}[Y_v^2]\preceq\frac{1}{\rho}\Pi_{K_v}$.
\end{proof}


\section{Empirical Evaluation}
\label{app:additional_exp}

\subsection{SuiteSparse benchmark}\label{sec:suitesparse}

\paragraph{Benchmark.}
We evaluate on 28 symmetric diagonally dominant M-matrices from the
SuiteSparse Matrix Collection, following the benchmark selection of
Gao, Kyng, and Spielman~\cite{AC}. The collection includes grid and
mesh Laplacians, finite-element and finite-volume discretizations,
shallow-water models, ill-conditioned structural systems, and
irregular graph problems, with the largest matrices containing up to
$4.8$ million nonzeros.

Ten matrices are diagonal or effectively diagonal, and every method
converges in one PCG iteration for every right-hand side. These
instances do not exercise the randomized clique estimator; their
timings primarily reflect fixed construction and application
overhead. We therefore retain them as correctness checks but exclude
them from aggregate performance comparisons.

Of the remaining 18 matrices, \texttt{bcsstm25} is reported as did
not finish (DNF) and excluded from the aggregates. Its conditioning
places the requested tolerance of $10^{-8}$ below the attainable
double-precision accuracy for some right-hand sides, causing every
method to reach the iteration limit on those right-hand sides. At
tolerance $10^{-5}$, all methods converge in one iteration,
indicating that the failure is attributable to the matrix--tolerance
pair rather than to a particular solver. The aggregate comparison
therefore contains $17$ matrices.

The SDDM inputs are reduced to Laplacian form using the standard
Gremban expansion~\cite{gremban1996combinatorial}. This preprocessing
is applied once per matrix and shared identically by all methods. 

\paragraph{Results.}

Tables~\ref{tab:tier1} and~\ref{tab:tier2} report the complete
per-matrix results for the base- and doubled-granularity comparisons,
\AC versus \CASTone and \ACtwo versus \CASTtwo, respectively.
Matrices are ordered by total workload time. Table~\ref{tab:trivial}
reports the ten trivial instances excluded from the
comparative aggregates.

\paragraph{Base-granularity comparison.}
\CASTone is faster than \AC on 14 of the 17 matrices, with an average
 speedup of $1.17\times$. The largest gains occur on
\texttt{nos7} ($1.55\times$), \texttt{shallow\_water2}
($1.54\times$), and \texttt{nos6} ($1.48\times$). On the two largest
matrices, \texttt{ecology1} and \texttt{ecology2}, \CASTone is
$1.25\times$ and $1.26\times$ faster, respectively.

Build costs are similar for the two methods on most matrices, whereas
solve costs generally favor \CASTone. The aggregate improvement
therefore arises primarily from the solve phase rather than from
cheaper factor construction. The exceptions are \texttt{apache1}
($0.73\times$), \texttt{jnlbrng1} ($0.77\times$), and
\texttt{bcsstm24} ($0.85\times$). On \texttt{apache1} and
\texttt{jnlbrng1},  \CASTone requires more PCG
iterations; \CASTone requires 54 iterations per solve on
\texttt{apache1}, compared with 28 for \AC.

\paragraph{Doubled-granularity comparison.}
\CASTtwo is faster than \ACtwo on 11 of the 17 matrices, with an average speedup of $1.069\times$. It performs particularly
well on \texttt{shallow\_water2} ($1.42\times$), the two
\texttt{ecology} matrices ($1.27$--$1.31\times$), and
\texttt{apache1} ($1.22\times$). Thus, the $\rho=2$ construction
recovers the loss observed for \CASTone on \texttt{apache1},
consistent with improved robustness under finer splitting.

On \texttt{jnlbrng1}, \ACtwo requires 14 iterations per solve,
compared with 27 for \CASTtwo, yielding a speedup ratio of
$0.64\times$. The largest loss occurs on \texttt{Andrews}
($0.50\times$), whose highly skewed degree distribution leads to a
\CASTtwo build cost of $2.09\,\mu\mathrm{s}/\mathrm{nnz}$, compared
with $0.63\,\mu\mathrm{s}/\mathrm{nnz}$ for \ACtwo. This
construction overhead makes \CASTone, rather than \CASTtwo, the
preferable CAST configuration on this matrix.

\paragraph{Cost decomposition and factor reuse.}
The performance differences reflect both iteration count and
preconditioner-application cost. On \texttt{ecology1}, for example,
\CASTtwo requires 26.3 iterations per solve, compared with 29.5 for
\ACtwo, while also having a lower per-iteration application cost.
Together, these effects produce the observed $1.31\times$
total-time speedup.

Factor reuse is particularly important for the doubled-granularity
variants. In a separate solve-count sweep, at $q=1$, \ACtwo is
$51\%$ more expensive than \AC and is slower on
all 17 matrices, while each CAST variant remains within a few percent
of its corresponding baseline. The additional construction cost is
amortized as the number of right-hand sides increases; at $q=250$,
the aggregate comparison modestly favors \CASTtwo.

\begin{table*}[!t]

\centering

\caption{Base-granularity comparison (\AC versus \CASTone) on the
17 nontrivial SuiteSparse matrices. $T_{250}$ is the total time for
factor construction and 250 PCG solves to relative residual
$10^{-8}$. Each entry is reported in seconds as the median over five
independently seeded factor draws, and the speedup is the ratio of
these medians, 
$T_{250}(\AC)/T_{250}(\CASTone)$. Build and per-right-hand-side solve
costs are normalized by the number of input nonzeros.}

\label{tab:tier1}

\small

\begin{tabular}{l r rr r rr rr}

\toprule

& & \multicolumn{2}{c}{$T_{250}$ (s)} & &

\multicolumn{2}{c}{solve $\mu$s/nnz} &

\multicolumn{2}{c}{build $\mu$s/nnz} \\

\cmidrule(lr){3-4}\cmidrule(lr){6-7}\cmidrule(lr){8-9}

Matrix & nnz & \AC & \CASTone & speedup &

\AC & \CASTone & \AC & \CASTone \\

\midrule

ecology2         & 4{,}995{,}991 & 227.96 & 181.51 & 1.26$\times$ & 0.174 & 0.139 & 0.12 & 0.12 \\

ecology1         & 4{,}996{,}000 & 226.49 & 181.90 & 1.25$\times$ & 0.171 & 0.138 & 0.12 & 0.12 \\

apache1          &   542{,}184 &  14.26 &  19.59 & 0.73$\times$ & 0.097 & 0.137 & 0.09 & 0.09 \\

Andrews          &   760{,}154 &  13.84 &  11.29 & 1.23$\times$ & 0.070 & 0.059 & 0.19 & 0.23 \\

shallow\_water2  &   327{,}680 &   5.46 &   3.55 & 1.54$\times$ & 0.059 & 0.044 & 0.07 & 0.10 \\

torsion1         &   197{,}608 &   3.80 &   3.29 & 1.16$\times$ & 0.074 & 0.066 & 0.07 & 0.08 \\

obstclae         &   197{,}608 &   3.79 &   3.24 & 1.17$\times$ & 0.073 & 0.065 & 0.07 & 0.08 \\

shallow\_water1  &   327{,}680 &   3.46 &   2.43 & 1.42$\times$ & 0.039 & 0.029 & 0.06 & 0.10 \\

jnlbrng1         &   199{,}200 &   3.35 &   4.36 & 0.77$\times$ & 0.067 & 0.087 & 0.07 & 0.07 \\

nopoly           &    70{,}842 &   1.48 &   1.29 & 1.14$\times$ & 0.081 & 0.073 & 0.05 & 0.06 \\

fv3              &    87{,}025 &   1.26 &   1.23 & 1.03$\times$ & 0.057 & 0.056 & 0.05 & 0.06 \\

fv2              &    87{,}025 &   0.72 &   0.64 & 1.12$\times$ & 0.033 & 0.029 & 0.05 & 0.05 \\

fv1              &    85{,}264 &   0.71 &   0.62 & 1.14$\times$ & 0.033 & 0.029 & 0.05 & 0.05 \\

nos7             &     4{,}617 &   0.102 &  0.066 & 1.55$\times$ & 0.086 & 0.055 & 0.07 & 0.08 \\

gr\_30\_30       &     7{,}744 &   0.085 &  0.078 & 1.09$\times$ & 0.043 & 0.039 & 0.05 & 0.06 \\

bcsstm24         &     3{,}562 &   0.084 &  0.100 & 0.85$\times$ & 0.087 & 0.103 & 0.04 & 0.07 \\

nos6             &     3{,}255 &   0.076 &  0.051 & 1.48$\times$ & 0.092 & 0.062 & 0.06 & 0.07 \\

\midrule

\multicolumn{4}{l}{Arithmetic mean of speedups}

& \textbf{1.172$\times$} & \multicolumn{4}{l}{wins 14/17} \\

\multicolumn{9}{l}{\emph{Excluded (DNF):} \texttt{bcsstm25} --- its conditioning places the $10^{-8}$ target below attainable double-precision accuracy,} \\
\multicolumn{9}{l}{ so a subset of right-hand sides never converges for any method.} \\

\bottomrule

\end{tabular}

\end{table*}

\begin{table*}[!t]

\centering

\caption{Doubled-granularity comparison (\ACtwo versus \CASTtwo) on
the same 17 SuiteSparse matrices and under the same protocol as
Table~\ref{tab:tier1}. Speedup is the ratio of these medians defined as
$T_{250}(\ACtwo)/T_{250}(\CASTtwo)$.}

\label{tab:tier2}

\small

\begin{tabular}{l r rr r rr rr}

\toprule

& & \multicolumn{2}{c}{$T_{250}$ (s)} & &

\multicolumn{2}{c}{solve $\mu$s/nnz} &

\multicolumn{2}{c}{build $\mu$s/nnz} \\

\cmidrule(lr){3-4}\cmidrule(lr){6-7}\cmidrule(lr){8-9}

Matrix & nnz & \ACtwo & \CASTtwo & speedup &

\ACtwo & \CASTtwo & \ACtwo & \CASTtwo \\

\midrule

ecology1         & 4{,}996{,}000 & 164.76 & 125.98 & 1.31$\times$ & 0.124 & 0.095 & 0.25 & 0.28 \\

ecology2         & 4{,}995{,}991 & 163.20 & 128.32 & 1.27$\times$ & 0.124 & 0.097 & 0.27 & 0.28 \\

Andrews          &   760{,}154 &  14.77 &  29.58 & 0.50$\times$ & 0.075 & 0.144 & 0.63 & 2.09 \\

apache1          &   542{,}184 &  12.01 &   9.85 & 1.22$\times$ & 0.082 & 0.069 & 0.24 & 0.23 \\

shallow\_water2  &   327{,}680 &   4.71 &   3.31 & 1.42$\times$ & 0.051 & 0.040 & 0.14 & 0.12 \\

obstclae         &   197{,}608 &   3.24 &   2.92 & 1.11$\times$ & 0.065 & 0.058 & 0.18 & 0.15 \\

torsion1         &   197{,}608 &   3.23 &   2.92 & 1.11$\times$ & 0.065 & 0.058 & 0.18 & 0.15 \\

shallow\_water1  &   327{,}680 &   2.95 &   2.45 & 1.21$\times$ & 0.035 & 0.030 & 0.10 & 0.12 \\

jnlbrng1         &   199{,}200 &   2.70 &   4.21 & 0.64$\times$ & 0.053 & 0.083 & 0.19 & 0.14 \\

nopoly           &    70{,}842 &   1.06 &   1.03 & 1.02$\times$ & 0.059 & 0.058 & 0.10 & 0.11 \\

fv3              &    87{,}025 &   0.94 &   0.98 & 0.96$\times$ & 0.043 & 0.044 & 0.12 & 0.15 \\

fv2              &    87{,}025 &   0.61 &   0.63 & 0.97$\times$ & 0.027 & 0.029 & 0.11 & 0.11 \\

fv1              &    85{,}264 &   0.60 &   0.61 & 0.99$\times$ & 0.027 & 0.028 & 0.12 & 0.11 \\

bcsstm24         &     3{,}562 &   0.085 &  0.099 & 0.85$\times$ & 0.086 & 0.103 & 0.07 & 0.07 \\

nos7             &     4{,}617 &   0.077 &  0.061 & 1.26$\times$ & 0.065 & 0.051 & 0.14 & 0.16 \\

gr\_30\_30       &     7{,}744 &   0.068 &  0.063 & 1.09$\times$ & 0.034 & 0.032 & 0.12 & 0.13 \\

nos6             &     3{,}255 &   0.061 &  0.049 & 1.25$\times$ & 0.074 & 0.059 & 0.08 & 0.10 \\

\midrule

\multicolumn{4}{l}{Arithmetic mean of speedups}

& \textbf{1.069$\times$} & \multicolumn{4}{l}{wins 11/17} \\

\multicolumn{9}{l}{\emph{Excluded (DNF):} \texttt{bcsstm25} --- its conditioning places the $10^{-8}$ target below attainable double-precision accuracy,} \\
\multicolumn{9}{l}{ so a subset of right-hand sides never converges for any method.} \\
\bottomrule

\end{tabular}

\end{table*}

\begin{table}[t]
\centering
\caption{The ten trivial SuiteSparse matrices, retained as
correctness checks and excluded from aggregate comparisons. Every
method converges in exactly one PCG iteration on every right-hand
side. The reported totals, in seconds, include factor construction
and 250 solves to relative residual $10^{-8}$ and therefore primarily
measure fixed construction and application overhead rather than
preconditioner quality.}
\label{tab:trivial}
\small
\begin{tabular}{l r cccc}
\toprule
Matrix & nnz & \AC & \ACtwo & \CASTone & \CASTtwo \\
\midrule
\texttt{bcsstm39}  & 46{,}772 & 0.141 & 0.142 & 0.153 & 0.153 \\
\texttt{t3dl\_e}   & 20{,}360 & 0.058 & 0.059 & 0.063 & 0.063 \\
\texttt{t2dal\_e}  &  4{,}257 & 0.012 & 0.012 & 0.013 & 0.013 \\
\texttt{bcsstm21}  &  3{,}600 & 0.012 & 0.012 & 0.012 & 0.012 \\
\texttt{bibd\_81\_2} & 3{,}240 & 0.010 & 0.010 & 0.011 & 0.011 \\
\texttt{bcsstm23}  &  3{,}134 & 0.009 & 0.008 & 0.009 & 0.010 \\
\texttt{bcsstm26}  &  1{,}922 & 0.006 & 0.006 & 0.006 & 0.006 \\
\texttt{bcsstm11}  &  1{,}473 & 0.005 & 0.004 & 0.005 & 0.005 \\
\texttt{bcsstm08}  &  1{,}074 & 0.003 & 0.003 & 0.003 & 0.003 \\
\texttt{bcsstm09}  &  1{,}083 & 0.003 & 0.003 & 0.004 & 0.004 \\
\bottomrule
\end{tabular}
\end{table}


\subsection{IPM Sequences on Chimera Graphs}
\label{sec:chimera}

\paragraph{Collection.}
We evaluate on the maximum-flow interior-point-method sequences from
the SDDM2023 benchmark suite of Gao, Kyng, and
Spielman~\cite{AC}. Each system is a weighted graph Laplacian with
$n=100{,}000$ vertices and arises from a Newton step of an
interior-point method for undirected maximum flow. The underlying
graphs are five independent draws, denoted
$i_1,\ldots,i_5$, from the \emph{Chimera} generator in
\texttt{Laplacians.jl}~\footnote{https://github.com/danspielman/Laplacians.jl}. Chimera graphs combine
heterogeneous structures, including grid-like components, star joins,
and graph products, and are designed to stress Laplacian solvers.

The five graph instances differ  in their elimination
geometry. Under \CASTone elimination, instance $i_1$ has mean pivot
degree $\bar d=26.7$ and maximum pivot degree
$d_{\max}=363$, whereas $i_3$ has
$\bar d=6.4$ and $d_{\max}=48$
(Table~\ref{tab:chimera250}). For each instance, the benchmark
provides IPM runs at five duality-gap targets,
\[
\varepsilon\in\{10^{-1},\ldots,10^{-5}\},
\]
with each run contributing between three and six Newton-step systems. The resulting collection contains $23$--$28$ matrix systems per instance and $128$ systems in total. Relative performance is stable across the duality-gap targets. We therefore aggregate over Newton steps and targets within each underlying graph instance.

\paragraph{Configuration.}
We use the same protocol, metrics, and hardware described in
Section~\ref{sec:empirical_methods}. Every timed solve satisfies the
explicitly verified residual criterion
\[
\frac{\lVert A\mathbf{x}-\mathbf{b}\rVert_2}
     {\lVert\mathbf{b}\rVert_2}
\leq 10^{-8},
\]
and no method reaches the iteration cap. Within each instance we report the mean total workload over its
systems and take speedups as ratios of these means, which measures
the aggregate cost of processing an entire IPM sequence. Win counts
give the complementary per-system view (Table~\ref{tab:chimera250}).

\paragraph{Results.}
\CASTone is faster than \AC on all $128$ systems (Tables ~\ref{tab:chimera250} and ~\ref{tab:chimera_nnz}.). The per-instance speedups range from $1.13\times$ to $1.24\times$,
and the per-system ratios range from $1.08\times$ to $1.31\times$.
The normalized solve cost is lower for \CASTone on all five
instances. Build costs are equal or nearly equal on $i_1$ and $i_2$
and higher for \CASTone on $i_3$--$i_5$; at $q=250$, however, the
reduction in solve cost dominates these construction differences.

The behavior of \CASTtwo depends more strongly on the
elimination-degree profile. On $i_3$, which has the smallest mean and
maximum pivot degrees, \CASTtwo is the fastest method: it is
$1.16\times$ faster than \ACtwo and wins on all $27$ systems. On
the other four instances, the \ACtwo/\CASTtwo ratios are
$0.35\times$, $0.76\times$, $0.45\times$, and $0.52\times$,
respectively.

This difference is associated with construction fill. At a degree-$d$
pivot, \CASTtwo samples one spanning tree on $2d$ auxiliary copies.
The tree contains $2d-1$ auxiliary edges before within-block edges
are discarded and parallel terminal-edge contributions are
aggregated, compared with $d-1$ edges for \CASTone. On instances
with high-degree pivots, the additional terminal-edge contributions
propagate through subsequent eliminations. Accordingly, the median \CASTtwo build cost relative to \CASTone
increases from $3.0\times$ on $i_3$ and $4.6\times$ on $i_2$ to
approximately $10\times$ on $i_4$ and $i_5$, and $13.1\times$ on
$i_1$. The doubled-granularity baseline is also structure dependent:
\ACtwo has lower mean total cost than \AC on $i_2$ and $i_3$, but
higher cost on $i_1$, $i_4$, and $i_5$.

\paragraph{Choosing the splitting factor.}
At $q=250$, \CASTtwo outperforms \CASTone only on $i_3$. Its factor
is $3.0\times$ more expensive to construct on this instance, so
\CASTone remains $1.38\times$ faster at $q=4$. The two variants
cross between $q=10$ and $q=50$, and \CASTtwo is $1.17\times$
faster by $q=250$. On the four instances with
$d_{\max}\geq92$, \CASTtwo is $1.3$--$4.6\times$ slower than
\CASTone at $q=250$. Within this collection, the results support a degree-aware empirical
choice of $\rho$. A light upper tail in the pivot-degree distribution
favors $\rho=2$ when the factor is reused sufficiently to amortize
its higher construction cost. In contrast, $\rho=1$ is preferable
when high-degree pivots make downstream fill and factor construction
dominant. The pivot-degree profile observed during a \CASTone
construction may inform this choice, although we treat this as an
empirical heuristic rather than as an evaluated adaptive selection
algorithm.


\begin{table*}[t]
\centering
\caption{Chimera-IPM total workload at $q=250$ right-hand sides per factor.
$T_{250}$ includes factor construction and 250 PCG solves to verified
relative residual $10^{-8}$. For each system--method pair we report
the median over five independently seeded factor draws; displayed
times are arithmetic means of these medians over the $23$--$28$
systems associated with each underlying graph instance. Parentheses give the number
of systems on which the corresponding \CAST{} variant is faster. } 
\label{tab:chimera250}
\small
\begin{tabular}{l r r r rrrr rr}
\toprule
& & & & \multicolumn{4}{c}{$T_{250}$ (s)} &
\multicolumn{2}{c}{Speedup (mean)} \\
\cmidrule(lr){5-8}\cmidrule(lr){9-10}
Instance & nnz & $\bar d$ & $d_{\max}$ &
\AC & \ACtwo & \CASTone & \CASTtwo &
$\AC/\CASTone$ &
$\ACtwo/\CASTtwo$ \\
\midrule
$i_1$ & 1{,}100{,}592 & 26.7 & 363 & 30.6 & 40.1 & 24.6 & 113.8 & $1.24\times$ (28/28) & $0.35\times$ (0/28) \\
$i_2$ & 797{,}974 & 9.5 & 92 & 19.8 & 17.6 & 17.5 & 23.1 & $1.13\times$ (24/24) & $0.76\times$ (0/24) \\
$i_3$ & 499{,}696 & 6.4 & 48 & 17.9 & 15.5 & 15.7 & 13.4 & $1.14\times$ (27/27) & $1.16\times$ (27/27) \\
$i_4$ & 814{,}436 & 15.2 & 229 & 23.1 & 25.8 & 19.5 & 57.4 & $1.19\times$ (26/26) & $0.45\times$ (0/26) \\
$i_5$ & 499{,}982 & 10.7 & 246 & 13.9 & 15.6 & 11.9 & 30.1 & $1.17\times$ (23/23) & $0.52\times$ (0/23) \\
\midrule
\multicolumn{10}{l}{Across instances: \AC/\CASTone arithmetic mean $\mathbf{1.18\times}$; \CASTone is faster on $128/128$ systems.} \\
\bottomrule
\end{tabular}
\end{table*}
\begin{table*}[t]
\centering
\caption{Chimera-IPM costs normalized by the number of input
nonzeros. The relative
construction cost of \CASTtwo broadly follows the upper tail of the
elimination-degree distribution: its build cost is $3.0\times$ that
of \CASTone on $i_3$, approximately $10\times$ on $i_4$ and $i_5$,
and $13.1\times$ on $i_1$. }
\label{tab:chimera_nnz}
\small
\begin{tabular}{l rrrr rrrr}
\toprule
& \multicolumn{4}{c}{Solve ($\mu\mathrm{s}/\mathrm{nnz}$)} &
\multicolumn{4}{c}{Build ($\mu\mathrm{s}/\mathrm{nnz}$)} \\
\cmidrule(lr){2-5}\cmidrule(lr){6-9}
Instance & \AC & \ACtwo & \CASTone & \CASTtwo &
\AC & \ACtwo & \CASTone & \CASTtwo \\
\midrule
$i_1$ & 0.105 & 0.139 & 0.087 & 0.392 & 0.40 & 1.15 & 0.41 & 5.37 \\
$i_2$ & 0.093 & 0.083 & 0.083 & 0.110 & 0.16 & 0.37 & 0.16 & 0.73 \\
$i_3$ & 0.133 & 0.115 & 0.116 & 0.099 & 0.10 & 0.32 & 0.15 & 0.45 \\
$i_4$ & 0.109 & 0.119 & 0.091 & 0.269 & 0.25 & 0.68 & 0.28 & 2.85 \\
$i_5$ & 0.101 & 0.114 & 0.089 & 0.230 & 0.18 & 0.59 & 0.28 & 2.80 \\
\bottomrule
\end{tabular}
\end{table*}


\begin{table*}[t]
\centering
\caption{Spielman-IPM total cost at $q=250$ right-hand sides per
factor. $T_{250}$ includes factor construction and $250$ PCG solves to
verified relative residual $10^{-8}$. Each scale has $10$ systems. For each system--method pair we
report the median over five independently seeded factor draws;
displayed times are arithmetic means of these medians over the systems
in each instance, and speedups are ratios of the displayed means, so they
measure the aggregate cost of solving an entire IPM sequence.
Parentheses report the number of systems on which the \CAST{} variant
is faster. \AC{} and \ACtwo{} use the buffered application kernel.}
\label{tab:spielman_tot}
\small
\begin{tabular}{l r rrrr rr}
\toprule
& & \multicolumn{4}{c}{$T_{250}$ (s)} &
\multicolumn{2}{c}{Speedup} \\
\cmidrule(lr){3-6}\cmidrule(lr){7-8}
Scale & nnz &
\AC & \ACtwo & \CASTone & \CASTtwo &
$\AC/\CASTone$ & $\ACtwo/\CASTtwo$ \\
\midrule
$k=100$ & 1{,}025{,}404 & 7.35 & 6.84 & 8.23 & 1.45 &
$0.89\times$ (9/10) & $4.72\times$ (10/10) \\
$k=200$ & 8{,}080{,}604 & 59.68 & 54.65 & 48.19 & 12.30 &
$1.24\times$ (9/10) & $4.44\times$ (10/10) \\
$k=300$ & 27{,}226{,}204 & 204.60 & 182.78 & 132.75 & 40.42 &
$1.54\times$ (10/10) & $4.52\times$ (10/10) \\
$k=400$ & 64{,}321{,}604 & 409.21 & 403.72 & 290.87 & 99.81 &
$1.41\times$ (10/10) & $4.04\times$ (10/10) \\
\midrule
\multicolumn{8}{l}{Across scales: $\AC/\CASTone$ mean $1.27\times$;
$\ACtwo/\CASTtwo$ mean $\mathbf{4.43\times}$, with \CASTtwo{} faster
on $40/40$ systems.} \\ \\
\bottomrule
\end{tabular}
\end{table*}

\begin{table*}[!ht]
\centering
\caption{Spielman IPM ladder normalized by the number of input nonzeros, under the same protocol as Table~\ref{tab:spielman_tot}.}
\label{tab:spielman_nnz}
\small
\begin{tabular}{l rrrr rrrr}
\toprule
& \multicolumn{4}{c}{solve $\mu$s/nnz} &
\multicolumn{4}{c}{build $\mu$s/nnz} \\
\cmidrule(lr){2-5}\cmidrule(lr){6-9}
Scale & \AC & \ACtwo & \CASTone & \CASTtwo &
\AC & \ACtwo & \CASTone & \CASTtwo \\
\midrule
$k=100$   & 0.0248 & 0.0247 & 0.0204 & 0.0054 & 0.032 & 0.061 & 0.049 & 0.049 \\
$k=200$   & 0.0248 & 0.0247 & 0.0157 & 0.0055 & 0.043 & 0.072 & 0.070 & 0.069 \\
$k=300$ & 0.0247 & 0.0246 & 0.0161 & 0.0055 & 0.042 & 0.103 & 0.073 & 0.077 \\
$k=400$ & 0.0247 & 0.0248 & 0.0168 & 0.0055 &
0.044 & 0.121 & 0.076 & 0.081 \\
\bottomrule
\end{tabular}
\end{table*}

\begin{table*}[!ht]
\centering
\caption{Sachdeva-star results at $q=250$ right-hand sides per factor.
Iteration counts are means over the $250$ solves, and $T_{250}$
includes factor construction and $250$ PCG solves to relative residual
$10^{-8}$. For every instance--method pair, each entry is the median
over three independently seeded factor draws. Speedups are ratios of
the displayed total times. The PCG iteration cap is $5\times10^{3}$
and is never reached.}
\label{tab:sachdeva}
\footnotesize
\begin{tabular}{l r rrrr rrrr rr}
\toprule
& & \multicolumn{4}{c}{Iterations per solve} &
\multicolumn{4}{c}{$T_{250}$ (s)} &
\multicolumn{2}{c}{Speedup} \\
\cmidrule(lr){3-6}\cmidrule(lr){7-10}\cmidrule(lr){11-12}
Scale & nnz &
\AC & \ACtwo & \CASTone & \CASTtwo &
\AC & \ACtwo & \CASTone & \CASTtwo &
$\frac{AC}{\CASTone}$ &
$\frac{ACtwo}{\CASTtwo}$ \\
\midrule
$k=50$ & 62{,}551 & 50.8 & 27.6 & 37.9 & 25.2 &
0.88 & 0.46 & 0.55 & 0.37 & $1.58\times$ & $1.23\times$ \\
$k=100$ & 500{,}101 & 108.9 & 37.9 & 68.3 & 28.9 &
15.8 & 5.49 & 7.77 & 3.32 & $2.04\times$ & $1.65\times$ \\
$k=200$ & 4{,}000{,}201 & 279.1 & 51.9 & 139.0 & 36.3 &
349 & 66.6 & 127 & 33.2 & $2.75\times$ & $2.00\times$ \\
$k=400$ & 32{,}000{,}401 & 693.4 & 58.6 & 287.9 & 42.8 &
7{,}519 & 638 & 1{,}986 & 298 & $3.79\times$ & $2.14\times$ \\
$k=600$ & 108{,}000{,}601 & 1{,}120.8 & 65.0 & 352.9 & 44.7 &
42{,}804 & 2{,}539 & 8{,}072 & 1{,}033 &
$5.30\times$ & $2.46\times$ \\
\midrule
\multicolumn{12}{l}{\CASTtwo is the fastest method on every instance and every solve count measured.} \\
\bottomrule
\end{tabular}
\end{table*}


\subsection{IPM Sequences on Spielman Graphs}
\label{sec:spielman}

\paragraph{Collection.}
The second IPM family from the SDDM2023 benchmark
suite~\cite{AC} forms a scaling ladder with one Spielman graph for each size
parameter $k\in\{100,200,300,400\}$. Each graph contributes the
Laplacians from the ten Newton steps of a maximum-flow
interior-point-method run, yielding $40$ systems in total. The systems range from
$n=3.4\times10^{5}$ to $2.1\times10^{7}$ vertices and contain up to
$6.4\times10^{7}$ nonzeros.

These graphs are structurally close to trees but have highly
heterogeneous edge weights. Most of their vertices have
degree at most two, while the number of edges beyond a spanning tree
ranges from $5{,}100$ to $80{,}400$. Each graph also contains $k$
hub vertices of input degree $2k$. Across the IPM trajectory, edge
weights span between $9.1$ and $14.3$ orders of magnitude, increasing primarily during the first approximately four Newton steps and then
stabilizing.

Despite the high input degrees of the hubs, minimum-degree elimination produces no large pivots. Path vertices are eliminated first, and every hub has degree at most five when it is eventually eliminated. The mean pivot degree is $2.00$ at every scale. Both CAST variants produce factors of the same measured size.

\paragraph{Configuration.}
We use the protocol and metrics of
Section~\ref{sec:empirical_methods}. Unless otherwise stated, each
factor is evaluated on $q=250$ right-hand sides. For every
system--method pair, the reported value is the median over five
independently seeded factor draws. For each scale ($10$ systems) we report the mean total workload over its systems
and take speedups as ratios of these means, which measures the
aggregate cost of processing an entire IPM sequence. Win counts give
the complementary per-system view. The larger $k\in\{500,600\}$ sequences, containing approximately
$1.5\times10^{8}$ and $2.2\times10^{8}$ nonzeros, exceed the memory
capacity of the benchmark machine for every method under the
multi-solve protocol and are therefore excluded.

\paragraph{$\rho=2$: fast iteration convergence.}
\CASTtwo converges in one to two PCG iterations on every recorded
solve across all four problem scales. The worst  verified
relative residual is $5.5\times10^{-9}$, which is below the requested
tolerance of $10^{-8}$. At $q=250$, the
available comparisons give mean speedups of
$4.04$--$4.72\times$ over \ACtwo.

The normalized solve cost of \CASTtwo remains between $0.0054$ and
$0.0055\,\mu\mathrm{s}/\mathrm{nnz}$ on the reported scales. The
improved convergence does not increase the measured factor size:
\CASTone and \CASTtwo both produce factors containing approximately
$0.66$ times as many nonzeros as the input, and their construction
costs are similar at the reported scales. 

This behavior is associated with the uniformly small pivot degrees.
At a degree-$d$ pivot, where $d\leq5$ throughout this family,
\CASTone samples a spanning tree with $d-1$ edges. In contrast,
\CASTtwo samples one spanning tree on $2d$ auxiliary copies,
containing $2d-1$ auxiliary edges before within-block edges are
discarded and parallel terminal-edge contributions are aggregated.
Two-way splitting therefore increases local sampling
granularity while remaining inexpensive when $d$ is small.

The low-iteration convergence is an empirical property of this graph
family. It does not imply that \CASTtwo reproduces every Schur clique
exactly, nor is \CASTtwo the union of two independent trees on the
original terminal neighborhood. Rather, the resulting factor is
sufficiently accurate for PCG to satisfy the requested tolerance
after one to two iterations on every tested Spielman system.

\paragraph{$\rho=1$: reliability on extreme-weight systems.}
The base $\rho=1$ estimator is less consistent on these
extreme-weight, near-tree systems, where an unlucky draw can produce a
noticeably worse factor. At $k=100$ this is visible in the aggregate,
which favors \AC{} despite \CASTone{} being faster on nine of ten
systems: one system dominates the scale total
(Table~\ref{tab:spielman_tot}).

\paragraph{Comparison with Chimera-IPM.}
The Spielman and Chimera families exhibit complementary behavior. On
the Chimera instances with heavy elimination-degree tails
(Section~\ref{sec:chimera}), the additional local contributions
introduced by $\rho=2$ propagate through subsequent eliminations and
substantially increase construction cost. On the Spielman family,
every pivot has degree at most five, so the same increase in local
sampling granularity remains inexpensive and yields convergence in one to two iterations.

Together, these results support a degree-aware empirical choice of
the splitting factor. A uniformly light elimination-degree
distribution favors $\rho=2$ when the factor is reused, whereas
$\rho=1$ is preferable when high-degree pivots make downstream fill
and construction cost dominant. We treat this as an empirical
selection heuristic rather than as an evaluated adaptive algorithm.



\subsection{Sachdeva-star stress tests}
\label{sec:sachdeva}
\paragraph{Collection.}
The Sachdeva-star family is a synthetic construction on which the \AC{} estimator is known to require rapidly increasing iteration counts~\cite{AC}. For an even parameter $k$, each instance consists of a hub connected by unit-weight gateway edges to $k/2$ disjoint cliques of size $k$. All edges have unit weight. The resulting Laplacian has
\[
n = 1+\frac{k^{2}}{2}
\qquad\text{and}\qquad
\mathrm{nnz} = \frac{k^{3}}{2}+k+1.
\]
We evaluate
\[
k\in\{50,100,200,400,600\},
\]
corresponding to $n=1{,}251$ through $180{,}001$ vertices and up to $108{,}000{,}601$ stored non\-zeros.

Unlike the Spielman family, whose difficulty arises from extreme weights on near-tree graphs, the Sachdeva-star instances are unweighted and structurally challenging. Eliminations within the clique blocks create dense local neighborhoods, and the quality of the base-granularity estimators deteriorates as $k$ increases. 

\paragraph{Configuration.}
We use the protocol and metrics of Section~\ref{sec:empirical_methods}. Each factor is evaluated on $q=250$ right-hand sides, and for every instance--method pair the reported value is the median over three independently seeded factor draws. Because iteration growth is a primary quantity of interest on this family, we raise the PCG iteration cap from $10^{3}$ to $5\times10^{3}$. No method reaches this cap.  Every solve satisfies the verified residual criterion
\[
\frac{\lVert A\,\mathbf{x}-\mathbf{b}\rVert_{2}}
{\lVert\mathbf{b}\rVert_{2}}
\leq 10^{-8}.
\]

\paragraph{Results.}
Table~\ref{tab:sachdeva} shows a clear separation between the base- and doubled-granularity variants. The mean iteration count of \AC increases from $50.8$ at $k=50$ to $1{,}120.8$ at $k=600$. \CASTone also becomes less effective as the instances grow, but its iteration count increases more slowly, from $37.9$ to $352.9$.
The doubled-granularity variants are substantially more stable. Across the same range, \ACtwo increases from $27.6$ to $65.0$ iterations per solve, while \CASTtwo increases from $25.2$ to $44.7$. At $q=250$, \CASTtwo is faster than \ACtwo on every instance, with speedup increasing monotonically from $1.23\times$ at $k=50$ to $2.46\times$ at $k=600$. Relative to \AC, its total-time advantage at the largest instance is
\[
\frac{42{,}804}{1{,}033} \approx 41.4\times.
\]

A separate solve-count sweep shows that \CASTtwo{} is already the
fastest method at $q=1$, so its advantage on this family does not
depend on amortizing construction cost.

\paragraph{Interpretation.}
Unlike on the Spielman family, \CASTtwo does not reduce these systems to one or two PCG iterations. The clique-dominated elimination structure remains nontrivial after contraction, and its mean iteration count grows from $25.2$ to $44.7$ along the ladder. Nevertheless, the growth is mild compared with the base-granularity methods. Although clique blocks produce high-degree pivots, the additional
contributions of $\rho=2$ fall within blocks that are already dense,
so they do not propagate as new fill --- the opposite of the Chimera
instances with heavy degree tails, where the extra contributions land
on sparse neighborhoods and compound through later eliminations. This behavior is consistent with the local second-moment improvement of the $\rho=2$ construction: finer splitting reduces local sampling variability, although the theory does not by itself imply a global iteration bound.
The Spielman and Sachdeva families therefore illustrate two distinct benefits of $\rho=2$. On the low-degree Spielman eliminations, \CASTtwo empirically produces an almost exact pre\-conditioner and converges in one to two iterations. On the clique-dominated Sachdeva instances, it does not eliminate iteration growth, but it limits that growth sufficiently to provide the best total time throughout the tested ladder. These results agree qualitatively with those reported in~\cite{AC}:
the method ordering and the growth of the base-granularity iteration
counts are unchanged, although absolute iteration counts here are
approximately $15\%$--$30\%$ higher, since our right-hand sides are
Gaussians projected onto $\mathbf 1^\perp$ rather than the range-restricted vectors used there.

\end{document}